\documentclass[11pt]{article}

\usepackage[in]{fullpage}

\usepackage{url}            
\usepackage{booktabs}       
\usepackage{amsfonts}       

\usepackage{subfig}

\usepackage{mathrsfs}

\usepackage[table,xcdraw]{xcolor}
\definecolor{airforceblue}{rgb}{0.36, 0.54, 0.66}
\definecolor{amber}{rgb}{1.0, 0.75, 0.0}
\definecolor{burntorange}{rgb}{0.8, 0.33, 0.0}
\definecolor{dimgray}{rgb}{0.41, 0.41, 0.41}

\usepackage[pagebackref = true,linkcolor={red!80!black},citecolor={blue!70!black}, urlcolor={blue!70!black},colorlinks, breaklinks]{hyperref}
\usepackage{tikz}
\usetikzlibrary{fit}
\tikzset{%
	highlight/.style={rectangle,blend mode = multiply,draw=blue!90!black,thick,rounded corners = 0.3 mm,inner sep=0.5pt}
}

\usepackage{microtype}
\usepackage{graphicx}
\usepackage{booktabs} 

\usepackage{amsmath,amssymb,amsthm}
\usepackage{bbold}
\usepackage[sans]{dsfont}

\usepackage{pifont}

\usepackage{thmtools}
\usepackage{thm-restate}

\usepackage{multirow}

\usepackage{tabularx}

\usepackage{textcomp}

\usepackage[shortlabels]{enumitem}

\usepackage{thmtools}

\declaretheoremstyle[qed=$\square$,parent=section]{definitionwithend}

\declaretheorem[style=definitionwithend]{theorem}
\declaretheorem[style=definitionwithend]{proposition}
\declaretheorem[style=definitionwithend]{definition}

\declaretheorem[style=definitionwithend]{remark}

\usepackage{tikz}
\usetikzlibrary{shadows}
\usetikzlibrary{arrows,positioning} 

\usepackage{contour}
\usepackage[normalem]{ulem}

\contourlength{0.8pt}

\numberwithin{equation}{section}

\usepackage{bm}

\usepackage{wrapfig}

\usepackage{authblk}
\newcommand{\ceil}[1]{\left\lceil #1 \right\rceil}

\newcommand{\PP}{\mathbb P}

\newcommand{\EE}{\mathbb E}
\newcommand{\R}{\mathbb R}

\newcommand{\E}[1]{{\mathbb E}\left[ #1 \right]}
\newcommand{\mc}{\mathcal}
\newcommand{\diff}{\mathrm{d}}

\DeclareMathOperator{\st}{s.t.}

\def\O#1{\text{\ding{\the\numexpr#1+171}}}

\usepackage{threeparttable}
\allowdisplaybreaks[4] 

\def\keywords{\vspace{1.5em}
{\textbf{Keywords}:\,\relax%
}}

\usepackage{todonotes}
\usepackage{natbib}
\usepackage{tcolorbox}
\newtcbox{\alertinline}[1][red]
  {on line, arc = 0pt, outer arc = 0pt,
    colback = #1!20!white, colframe = #1!50!black,
    boxsep = 0pt, left = 1pt, right = 1pt, top = 2pt, bottom = 2pt,
    boxrule = 0pt, bottomrule = 1pt, toprule = 1pt}
\usepackage{newtxtext}
\newtcolorbox{textbox}[1]{
    sharp corners,
    boxsep=0mm,
    toptitle=2mm,
    lefttitle=0mm,
    colframe=black!3,
    colback=black!3,
    title={\rule[-2pt]{4.5pt}{10pt}\hspace*{1.5mm}#1},
    fonttitle=\bfseries\itshape\sffamily,
    coltitle=black,
    halign=flush left,
}

\usepackage{subfloat}
\usepackage{subfig}
\usepackage{microtype}
\usepackage{graphicx}
\usepackage{booktabs} 

\usepackage{url}            
\usepackage{booktabs}       
\usepackage{amsfonts}       
\usepackage{nicefrac}       
\usepackage{microtype}      
\usepackage{xcolor}         
\usepackage{amsmath}
\usepackage{amsthm}
\usepackage{amssymb}
\usepackage{mathrsfs}
\usepackage{algorithm}
\usepackage{algorithmic}
\usepackage{graphicx}

\usepackage{bbding}
\usepackage{multirow}
\usepackage{dsfont}
\usepackage{booktabs}
\usepackage{multirow}
\usepackage[titletoc]{appendix}
\usepackage{lipsum}

\usepackage{setspace}

\usepackage{xcolor}
\definecolor{airforceblue}{rgb}{0.36, 0.54, 0.66}
\definecolor{amber}{rgb}{1.0, 0.75, 0.0}
\definecolor{burntorange}{rgb}{0.8, 0.33, 0.0}
\definecolor{dimgray}{rgb}{0.41, 0.41, 0.41}

\usepackage{amsmath}
\usepackage{amssymb}
\usepackage{mathtools}
\usepackage{amsthm}
\usepackage{enumitem}
\usepackage{statistics}

\begin{document}
\title{{
\bf Stability Evaluation via Distributional Perturbation Analysis}}

\author[1]{Jose Blanchet\thanks{\href{mailto:jose.blanchet@stanford.edu}{jose.blanchet@stanford.edu}}}
\author[2,3]{Peng Cui\thanks{\href{cuip@tsinghua.edu.cn}{cuip@tsinghua.edu.cn}}}
\author[1]{Jiajin Li\thanks{\href{mailto:jiajinli@stanford.edu}{jiajinli@stanford.edu}}}
\author[1,2]{Jiashuo Liu\thanks{\href{mailto:liujiashuo77@gmail.com}{liujiashuo77@gmail.com, jiashuo@stanford.edu}\\\hspace*{1.2em} Authors ordered alphabetically.}}
\affil[1]{Department of Management Science and Engineering, Stanford University}
\affil[2]{Department of Computer Science and Technology, Tsinghua University}
\affil[3]{Zhongguancun Lab}

\date{\today}
   
\date{\today}

\maketitle

\vspace{-5mm}

\begin{abstract}
The performance of learning models often deteriorates when deployed in out-of-sample environments. To ensure reliable deployment, we propose a stability evaluation criterion based on distributional perturbations. Conceptually, our stability evaluation criterion is defined as the minimal perturbation required on our observed dataset to induce a prescribed deterioration in risk evaluation. In this paper, we utilize the optimal transport (OT) discrepancy with moment constraints on the \textit{(sample, density)} space to quantify this perturbation. Therefore, our stability evaluation criterion can address both \emph{data corruptions} and \emph{sub-population shifts} --- the two most common types of distribution shifts in real-world scenarios. To further realize practical benefits, we present a series of tractable convex formulations and computational methods tailored to different classes of loss functions. The key technical tool to achieve this is the strong duality theorem provided in this paper. Empirically, we validate the practical utility of our stability evaluation criterion across a host of real-world applications. These empirical studies showcase the criterion's ability not only to compare the stability of different learning models and features but also to provide valuable guidelines and strategies to further improve models.

\end{abstract}


\section{Introduction}
The issue of poor out-of-sample performance frequently arises, particularly in high-stakes applications such as healthcare~\citep{bandi2018detection, wynants2020prediction,roberts2021common}, economics~\citep{hand2006classifier, ding2022retiring}, self-driving~\citep{malinin2021shifts,hell2021monitoring}.
This phenomenon can be attributed to discrepancies between the training and test datasets, influenced by various factors. Some of these factors include measurement errors during data collection~\citep{jacobucci2020machine, elmes2020accounting}, deployment in dynamic, non-stationary environments~\citep{camacho2011manipulation, conger2023strategic}, and the under-representativeness of marginalized groups in the training data~\citep{corbett2023measure}, among others. The divergence between training and test data presents substantial challenges to the reliability, robustness, and fairness of machine learning models in practical settings.
Recent empirical studies have shown that  algorithms intentionally developed for addressing distribution shifts—such as distributionally robust optimization~\citep{blanchet2019robust,sagawa2019distributionally,kuhn2019wasserstein, duchi2021learning, rahimian2019distributionally, blanchet2024distributionally}, domain generalization~\citep{zhou2022domain}, and causally invariant learning~\citep{arjovsky2019invariant, krueger2021out} --- experience a notable performance degradation  when faced with real-world scenarios~\citep{gulrajani2020search, frogner2021incorporating, yang2023change, liu2023need}. 

Instead of providing a robust training algorithm, we shift focus towards a more fundamental (in some sense even simpler) question:

\fbox{\begin{minipage}[t]{0.95\textwidth}
\begin{center}
\vspace{1mm}
\bf Q: How do we evaluate the stability of a learning model when subjected to data perturbations?
\vspace{1mm}
\end{center}
\end{minipage}}
\vspace{0.5mm}

To answer this question, our initial step is to gain a comprehensive understanding of  various types of data perturbations. In this paper, we categorize data perturbations into two classes: (\rm{i}) \emph{Data corruptions}, which encompass changes in the distribution support (i.e., observed data samples). These changes can be attributed to measurement errors in data collection or malicious adversaries. Typical examples include factors like street noises in speech recognition~\citep{kinoshita2020improving}, rounding errors in finance~\citep{li2015rounding}, adversarial examples in vision~\citep{goodfellow2014explaining} and, the Goodhart’s law empirically observed in government assistance allocation~\citep{camacho2011manipulation}.
(\rm{ii}) \emph{{{Sub-population shifts}}}, refer to perturbation on the probability density or mass function while keeping the same support. For example, model performances substantially degrade under demographic shifts in recommender systems~\citep{blodgett2016demographic, sapiezynski2017academic}; under temporal shifts in medical diagnosis~\citep{pasterkamp2017temporal}; and under spatial shifts in wildlife conservation~\citep{beery2021iwildcam}.


Recent investigation on the question \textbf{Q} predominantly centers around sub-population shifts, see \citep{li2021evaluating, namkoong2022minimax, gupta2023the}. 
However, in practical scenarios, it is common to encounter both types of data perturbation. Studies such as \citet{gokhale2022generalized} and \citet{zou2023on} have documented that models demonstrating robustness against sub-population shifts can still be vulnerable to data corruptions. This underscores the importance of adopting a more holistic approach when evaluating model stability, one that addresses both sub-population shifts and data corruptions.



To fully answer the question \textbf{Q}, we frame the model stability as a projection problem over probability space under the OT discrepancy with moment constraints. Specifically, we seek the minimum perturbation necessary on our reference measure (i.e., observed data) to guarantee that the model's risk remains below a specified threshold. The crux of our approach is to conduct this projection within the joint \emph{(sample, density)} space. Consequently, our stability metric is capable of addressing both data corruptions on the sample space  and sub-population shifts on the density or probability mass space. To enhance the practical utility of our approach,  we present a host of tractable convex formulations and computational methods tailored to different learning models. The key technical tool for this is the strong duality theorem provided in this paper.

To offer clearer insights, we visualize the most sensitive distribution in stylized examples. Our approach achieves a balanced and reasoned stance by avoiding overemphasis on specific samples or employing overly aggressive data corruptions.
Moreover, we demonstrate the practical effectiveness of our proposed stability evaluation criterion by applying it to tasks related to income prediction, health insurance prediction, and COVID-19 mortality prediction. These real-world scenarios showcase the framework's capacity to assess stability across various models and features, uncover potential biases and fairness issues, and ultimately enhance decision-making.
 

\vspace{1em}
\noindent \textbf{Notations.}\quad Throughout this paper, we let $\mathbb R$ denote the set of real numbers, $\mathbb R_+$ denote the subset of non-negative real numbers.
We use capitalized letters for random variables, e.g., $X,Y,Z$, and script letters for the sets, e.g., $\mathcal X, \mathcal Y, \mathcal Z$.
For any close set $\mathcal Z \subset \mathbb R^d$, we define $\mathcal P(\mathcal Z)$ as the family of all Borel probability measures on $\mathcal Z$.
For $\P \in \mathcal P(\mathcal Z)$, we use the notation $\mathbb E_\P[\cdot]$ to denote expectation with respect to the probability distribution $\P$.
For the prediction problem, the random variable of data points is denoted by $Z=(X,Y)\in \mathcal Z$, where $X\in \mathcal X$ denotes the input covariates, $Y\in\mathcal Y$ denotes the target. 
$f_\beta:\mathcal X\rightarrow \mathcal Y$ denotes the prediction model parameterized by $\beta$.
The loss function is denoted as $\ell:\mathcal Y \times \mathcal Y \rightarrow \mathbb R_+$, and $\ell(f_\beta(X),Y)$ is abbreviated as $\ell(\beta,Z)$. We use $(\cdot)_+=\max(\cdot,0)$. We adopt the conventions of extended arithmetic, whereby $\infty \cdot 0=0 \cdot \infty=0 / 0=0$ and $\infty-\infty=-\infty+\infty=1 / 0=\infty$.

\section{Model Evaluation Framework}
In this section, we present a stability evaluation criterion based on OT discrepancy with moment constraints, capable of considering both types of data perturbation --- data corruptions and sub-population shifts --- in a unified manner. 
The key insight lies in computing the projection distance, as shown in Figure \ref{fig:projection}, which involves minimizing the probability discrepancy between the most sensitive distribution denoted as $\Q^\star$ and the lifted training distribution $\P_{0}\otimes\delta_1$ in the joint (sample, density) space, while maintaining the constraint that the model performance falls below a specific threshold.
This threshold refer to a specific level of risk, error rate, or any other relevant performance metrics. 
The projection type methodology has indeed been employed in the literature for statistical inference, particularly in tasks like constructing confidence regions~\citep{owen2001empirical, blanchet2019robust}. However, this application is distinct from our current purpose.

\begin{figure}
    \centering
    \includegraphics[width=0.5\linewidth]{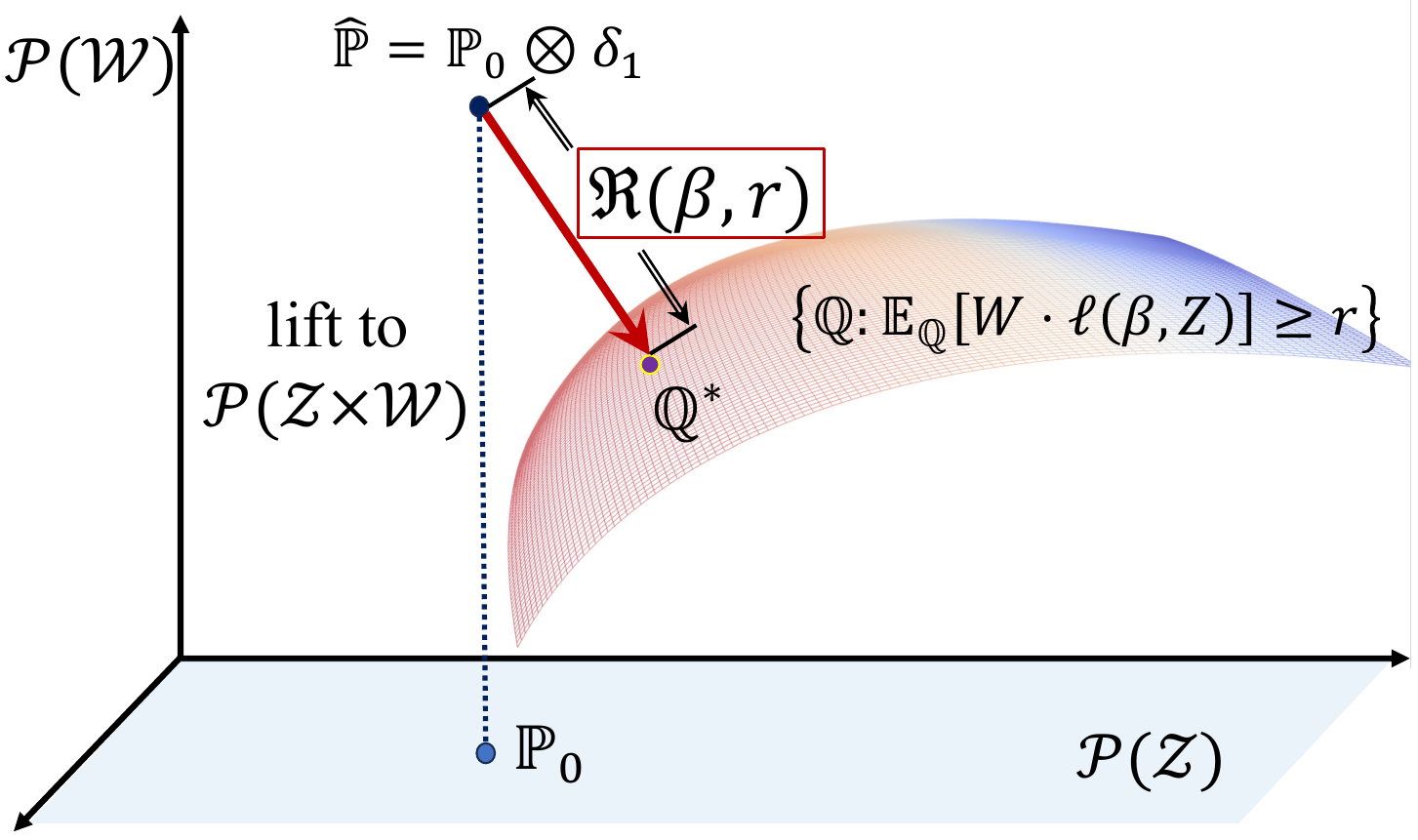}
    \caption{Data Distribution Projection}
    \label{fig:projection}
\end{figure}

\subsection{OT-based stability evaluation criterion}
\label{subsec:framework} 
 We begin by presenting the OT discrepancy with moment constraints, as proposed in ~\citet[Definition 2.1]{blanchet2023unifying}. This serves as a main technical tool for our further discussions.

\begin{definition}[OT discrepancy with moment constraints]
\label{def:OT}
If $\mc Z\subseteq \R^{d}$ and $\mc W\subseteq \R_+$ are convex and closed sets,  $c: (\mc Z \times \mc W)^2 \rightarrow  \R_+$ is a lower semicontinuous function, and $\Q,\P\in\mc P(\mc Z\times \mc W)$, then the OT discrepancy with  moment constraints induced by $c$, $\Q$ and $\P$ is the function $\mathds M_c: \mc P(\mc Z \times \mc W)^2 \to \R_+$ defined through
\begin{equation*}
\mathds M_c(\Q,\P) = \left\{
\begin{array}{clll}
\inf &\EE_{\pi}[c((Z , W), (\hat Z, \hat W))]\\
\st  &\pi \in \mc P((\mc Z \times \mc W)^2)\\[-0.7ex]
& \pi_{(Z, W)}  = \Q,\ \pi_{(\hat Z, \hat W)}  = \P \\
& \EE_\pi[W] = 1 \quad \pi \text{-a.s,}
\end{array}\right.
\label{def:mot}
\end{equation*}
where $\pi_{(Z, W)}$ and $\pi_{(\hat Z, \hat W)}$ are the marginal distributions of $(Z, W)$ and $(\hat Z, \hat W)$ under~$\pi$. 
\end{definition}

\begin{remark}
The core idea is to lift the original sample space $\mc Z$ to a higher dimensional space $\Zscr\times\Wscr$ ---- a joint (sample, density) space. Here, we treat the additional random variable $W$ as the ``density " or ``probability mass", making it also amenable to perturbations through optimal transport methods. However, these perturbations are subject to the constraint that the expectation of the density must remain equal to one. Thus, the  transportation cost function $c((z,w),(\hat z,\hat w))$ can measure the changes in both samples ($\hat z\rightarrow z$) and their probability densities ($\hat w\rightarrow w$).
\end{remark}

To evaluate the stability of a given learning model $f_\beta$ trained on the distribution $\P_{0}\in\mathcal P(\mathcal Z)$,
we formally introduce the  OT-based stability evaluation criterion as

\begin{equation}
\begin{aligned}
	\mathfrak R(\beta,r)= \left\{
        \begin{array}{cl} \inf \limits_{\Q\in \mathcal{P}(\mc Z \times \mc W)} & \,\, \mathds M_c(\Q,\hat \P)\\
        \,\text{s.t. } & \,\,  \mathbb E_{\Q}[W\cdot \ell(\beta,Z)]\geq r. \end{array}\right. 
\end{aligned}
\label{eq:primal}
\tag{P}
\end{equation}
Here, the reference measure $\hat \P$ is selected as $\P_{0}\otimes \delta_1$, with $\delta_1$ denoting the Dirac delta function,\footnote{This implies that the sample weights are almost surely equal to one with respect to the reference distribution, as we lack any prior information about them.} $\mathds M_c(\Q,\hat \P)$ represents the OT discrepancy with moment constraints between the projected distribution $\Q$ and the reference distribution $\hat \P$, $\ell(\beta,z)$ denotes the prediction risk of model $f_\beta$ on sample $z$, and $r> 0$ is the pre-defined risk threshold.

To sum up, we evaluate a model's stability under distribution shifts by quantifying the minimum level of perturbations required for the model's performance to degrade to a predetermined risk threshold. The magnitude of perturbations is determined through the use of the OT discrepancy with moment constraints and the cost function $c$, see definition \ref{def:OT}. 

Then, a natural question arises: How do we select the cost function $c$ to effectively quantify the various types of perturbations?
We aim for this cost function to be capable of quantifying changes in both the support of the distribution and the probability density or mass function. One possible candidate cost function is 
\begin{equation}
    c((z,w),(\hat z, \hat w)) = \theta_1 \cdot w\cdot d(z,\hat z) + \theta_2 \cdot (\phi(w)-\phi(\hat w) )_+. 
    \label{eq:cost_f}
\end{equation}
Here, $d(z,\hat z) = \|x-\hat x\|^2_2 + \infty \cdot |y-\hat y|$ quantifies the cost associated with the different data samples $z$ and $\hat z$ in the set $\mc Z$, with the label measurement's reliability considered infinite; $(\phi(w)-\phi(\hat w))_+$ denotes the cost related to differences in probability mass, where $\phi:\R_+\rightarrow \R_+$ is a convex function satisfying $\phi(1)=0$; $\theta_1, \theta_2 \ge 0$ serve as two hyperparameters, satisfying $1/{\theta_1}+1/{\theta_2}=C$ for some \emph{constant} $C$,  to control  the trade-off between the cost of perturbing the distribution's support and the probability density or mass on the observed data points.
This cost function was originally proposed in \citet[Section 5]{blanchet2023unifying} within the framework of distributionally robust optimization.

\begin{remark}[Effect of $\theta_1$ and $\theta_2$]
    \rm{(i)}
When $\theta_1 = +\infty$, the stability criterion $\mathfrak R(\beta,r)$ only counts the sub-population shifts, as any data sample corruptions are not allowed. In this scenario, our proposed stability criterion can be reduced to the one recently introduced in~\citet{gupta2023the} and~\citet{namkoong2022minimax}.  
\rm{(ii)} When $\theta_2= + \infty$, the stability criterion 
 $\mathfrak R(\beta,r)$ only takes the data corruptions into account instead. \rm{(iii)} 
The most intriguing scenario arises when both $\theta_1$ and $\theta_2$ have finite values. These parameters, $\theta_1$ and $\theta_2$, hold a pivotal role in adjusting the balance between data corruptions and sub-population shifts within our stability criterion, which allows us to simultaneously consider both types of distribution shifts. By manipulating the values of $\theta_1$ and $\theta_2$, we can achieve a versatile representation of a model's resilience across a wide range of distributional perturbation directions.
This adaptability carries significant implications when evaluating the robustness of models in diverse and ever-evolving real-world environments.
\end{remark}

\subsection{Dual reformulation and its interpretation}
Problem \eqref{eq:primal} constitutes an infinite-dimensional optimization problem over probability distributions and thus appears to be intractable. However, we will now demonstrate 
that by first establishing a strong duality result, problem \eqref{eq:primal} can be reformulated as a finite-dimensional optimization problems and discuss the structure of the most sensitive distribution from problem \eqref{eq:primal}.   

\begin{theorem}[Strong duality for problem \eqref{eq:primal}] 
\label{thm:duality}
Suppose that 
\textrm{(i)} The set $\mc Z\times \mc W$ is compact, 
\textrm{(ii)} $
\ell(\beta,\cdot) $ is upper semi-continuous for all $\beta$,  \textrm{(iii)} the cost function $c: (\mc Z \times \mc W)^2 \rightarrow \R_+$ is  continuous;
and \textrm{(iv)} the risk level $r$ is less than the worst case value $ \bar r:=\max_{z\in \mc Z} \ell(\beta,z)$. 
Then we have, 
\begin{equation}
	\mathfrak R(\beta,r) = \sup_{h \in \R_+, \alpha \in \R} hr +\alpha +\\ \EE_{\hat \P}\left[\tilde{\ell}_{c}^{\alpha, h}(\beta,(\hat Z,\hat W) )\right]
 \tag{D}
 \label{eq:dual}
\end{equation}
where the surrogate function $\tilde{\ell}_{c}^{\alpha,h}(\beta,(\hat z,\hat w) )$ equals to 
\[
\min_{(z,w) \in \mc Z \times \mc W}  c((z,w),(\hat z, \hat w)) +\alpha w - h\cdot  w \cdot \ell(\beta,
 z),
\]
for all $\hat z \in \mc Z$ and $\hat w\in \mc W$. 
\end{theorem}
For a detailed proof, we direct interested readers to the Appendix \ref{appendix-sec:strong-duality}. 

\begin{remark}
When the reference measure $\P_0$ is a discrete measure, some technical conditions in Theorem \ref{thm:duality} (e.g., compactness, (semi)-continuity) can be eliminated by utilizing the abstract semi-infinite duality theory for conic linear programs. Please refer to \citet[Proposition 3.4]{shapiro2001duality} and our proof in Appendix \ref{appendix-sec:strong-duality} for more detailed information. 
\end{remark}




If we adopt the cost function in the form of~\eqref{eq:cost_f} for two commonly used $\phi$ functions, we can simplify the surrogate function further by obtaining the closed form of $w$. Here, we explore the following cases:
(i) Selecting $\phi(t) = t\log t - t + 1$, which is associated with the Kullback--Leibler (KL) divergence.
(ii) Choosing $\phi(t) = (t - 1)^2$, which is linked to the $\chi^2$-divergence.

 \begin{proposition}[Dual reformulations]
 \label{theorem:dual-reformulation}
Suppose that $\mc W =\R_+$. 
\textrm{(i)} If $\phi(t) = t \log t - t + 1$, then the dual problem \eqref{eq:dual} admits:
\begin{equation}
		\sup_{h\geq 0} \, hr- \theta_2 \log \mathbb E_{\P_0} \left[\exp\left(\frac{\ell_{h,\theta_1}(\hat Z)}{\theta_2}\right) \right];
    \label{equ:D-risk-kl}
\end{equation}
\textrm{(ii)} If $\phi(t) = (t-1)^2$, then the dual problem \eqref{eq:dual} admits:
\begin{equation}
	\label{equ:D-risk-chi}
	\begin{aligned}
	\sup\limits_{h\geq 0, \alpha\in\mathbb R}  hr+\alpha+\theta_2 -\theta_2\mathbb E_{\P_0}\left[\left(\frac{\ell_{h,\theta_1}(\hat Z)+\alpha}{2\theta_2}+1\right)_+^2 \right],
	\end{aligned}
	\end{equation}
 
where the $d$-transform of $h\cdot \ell(\beta,\cdot)$ with the step size $\theta_1$ is defined as
 \begin{equation*}
	\ell_{h, \theta_1}(\hat z)\coloneqq \max_{z\in\Zscr} h\cdot  \ell(\beta,z)-\theta_1\cdot d(z,\hat z). 
\end{equation*}
\end{proposition}

When the reference measure $\P_0$ is represented as the empirical measure $\P_0 = \frac{1}{n}\sum_{i=1}^n\delta_{\hat z_i}$, the characterization of the most sensitive distribution $\Q^\star$, can be elucidated through the dual formulation provided in \eqref{equ:D-risk-kl} and \eqref{equ:D-risk-chi}.
\begin{remark}[Structure of the most sensitive distribution]
\label{remark:q-star}
 We express $\Q^\star$ as follows: 
 $\Q^\star = \tfrac{1}{n}\sum_{i=1}^n \delta_{(z_i^\star, w_i^\star)}$, where each $(z_i^\star, w_i^\star) \in \mathcal{Z} \times \mathbb{R}_+$ satisfies the conditions:  
\begin{align*}
		& z_i^\star = \mathop{\arg\max}_{z\in\Zscr} h^\star\ell(\beta;z)-\theta_1 \cdot d(z,\hat{z}_i), \quad \forall i \in [n]. 
\end{align*}

Using various $\phi$ functions requires adjusting the weight in a distinct manner:\\ \textrm{(i)} If $\phi(t) = \log t - t + 1$, then we have:
 \begin{align*}
  & w_i^\star \propto \exp\left(\frac{\ell_{h^\star,\theta_1}(\hat{z}_i)}{\theta_2}\right),\quad \forall i \in [n];  
  \end{align*}
   \textrm{(ii)} If $\phi(t) = (t-1)^2$, then we have:
   \[
w_i^\star \propto \left(\frac{\ell_{h^\star,\theta_1}(\hat{z}_i)-\alpha^*}{2\theta_2}+1  \right)_+,\quad \forall i \in [n], 
\]
 where $h^\star$ and $\alpha^\star$
are the optimal solution of problem \eqref{eq:dual}.
Therefore, it becomes evident that the most sensitive distribution encompasses both aspects of shifts: the transformation from $\hat{z}_i$ to $z_i^\star$ and the reweighting from $\frac{1}{n}$ to $w_i^\star$. Our cost function enables a versatile evaluation of model stability across a range of distributional perturbation directions. This approach yields valuable insights into the behavior of a model in different real-world scenarios and underscores the importance of incorporating both types of distributional perturbation in stability evaluation.
\end{remark}

\subsection{Computation}
In this subsection, our emphasis lies in addressing problems \eqref{equ:D-risk-kl} and \eqref{equ:D-risk-chi} with varying types of loss functions, specifically when the reference measure $\P_0$ takes the form of the empirical distribution.\\  

\noindent\textbf{Convex piecewise linear loss functions.}\quad If the loss function $\ell(\beta,\cdot)$ is piecewise linear (e.g., linear SVM), we can show that \eqref{equ:D-risk-kl} admits a tractable finite convex program. 

\begin{theorem}[KL divergence]
\label{thm:kl_linear}
    Suppose that $\mc Z =\mc \R^d\times\{+1,-1
    \}$ and  $\ell(\{(a_k,b_k)\}_{k\in[K]},z) = \max_{k \in [K]} y\cdot a_k^\top x+b_k$. The negative optimal value of problem \eqref{equ:D-risk-kl} is equivalent to the optimal value of the finite convex program:
\begin{flalign*}
\begin{array}{lllll}
     &\min\ \   &-h r  +t & \\
     &\st \ \   & \lambda \in \R_+, t \in \R, \eta \in \R_+^n, p\in\R_n\\
      && (\eta_i,  \theta_2, p_i- t) \in \mc K_{\exp} &\forall i \in [n] \\
     && \frac{\|a_k\|_2^2}{4\theta_1} h^2  + \hat y_i \cdot a_k^T\hat x_i \cdot h +b_k\leq p_i,\,  &~\forall k\in[K], \forall i \in [n]\\
         &&\frac{1}{n}\sum_{i=1}^n \eta_i \leq  \theta_2, 
    \end{array}
\end{flalign*}

where the set $\mc K_{\exp}$ is the exponential cone defined as 
\begin{equation*}
K_{\exp} = \left\{(x_1,x_2,x_3)\in \R^3: x_1\ge x_2 \cdot \exp\left(\tfrac{x_3}{x_2}\right),x_2>0\right\} \cup \{(x_1,0,x_3)\in\R^3: x_1\ge 0, x_3\leq 0\}. 
\end{equation*}
\end{theorem}

\begin{theorem}[$\chi^2$ Divergnce]
\label{thm:chi_linear}
    Suppose that $\mc Z =\mc \R^d\times\{+1,-1
    \}$ and  $\ell(\{(a_k,b_k)\}_{k\in[K]},z) = \max_{k \in [K]} y\cdot a_k^\top x+b_k$. The negative optimal value of problem \eqref{equ:D-risk-kl} is equivalent to the optimal value of the finite convex program
\begin{align*}
\begin{array}{cclll}
         &\min  &-h r  +t & \\
         &\st   & h \in \R_+, \alpha \in \R, t \in \R, \eta \in \R_+^n\\
         &&   \frac{\|a_k\|_2^2}{4\theta_1}\cdot h^2  + \hat y_i \cdot a_k^T\hat x_i \cdot h +b_k+{2\theta_2}\alpha +2\theta_2 \leq 2\theta_2 \eta_i\,  &~\forall k\in[K], \forall i \in [n]  \\
           && \frac{\theta_2}{n} \sum_{i=1}^n \eta_i^2  \leq t. \nonumber  
    \end{array}
\end{align*}
\end{theorem}

For a detailed proof, we direct interested readers to the Appendix \ref{subsec:appendix-linear} and \ref{subsec:proof-chi} for more detailed information.
Equipped with Theorem \ref{thm:kl_linear} and  \ref{thm:chi_linear}, we can calculate our evaluation criterion by general purpose conic optimization solvers such as MOSEK and GUROBI.\\

\noindent\textbf{0/1 loss function.}\quad 
In practical applications, employing a 0/1 loss function offers users a simpler method to set up the risk level $r$, which corresponds to a pre-defined acceptable level of \emph{error rate}. That is, given a trained model $\beta$, we define the loss function on the sample $(x,y)$ as 
\[
\ell(\beta, (x,y)) = \mathbb{I}_{y \neq f_\beta(x)}, 
\]
where $\mathbb{I}$ is the indicator function defined as $\mathbb{I}_{y \neq f_\beta(x)} =0$ if $y \neq f_\beta(x)$; $=0$ otherwise. In this scenario, the $d$-transform of $h\cdot \ell_\beta(\cdot)$ can be expressed in a closed form. Conceptually, this loss function promotes long-haul transportation, as it encourages either minimal perturbation or no movement at all, i.e., 
\[
\ell_{h, \theta_1}(\hat z) = (h-\theta_1\cdot d^\star (\hat z))_+,
\]
where $d^\star (\hat z):= \min_{z\in Z}\{d(z,\hat z):\ell(\beta,z)=1\}$. 
This distance quantifies the minimal adjustment needed to fool or mislead the classifier's prediction for the sample $\hat z$. A similar formulation has been employed in \citet{si2021testing} to assess group fairness through optimal transport projections. Finally,  the dual formulation \eqref{equ:D-risk-kl} is reduced to an one-dimensional convex problem w.r.t $h$. \\

\noindent\textbf{Nonlinear loss functions.}\quad 
For general nonlinear loss functions, such as those encountered in deep neural networks, the dual formulation \eqref{equ:D-risk-kl} retains its one-dimensional convex nature with respect to $h$. However, the primary computational challenge lies in solving the inner maximization problem concerning the sample $z$. In essence, this dual maximization problem \eqref{equ:D-risk-kl} for nonlinear loss functions is closely associated with adversarial training~\citep{nouiehed2019solving, yi2021improved}. All algorithms available in the literature for this purpose can be applied to our problem as well. The key distinction lies in the outer loop. In our case, we optimize over $h\in\mathbb{R}_+$ to perturb the sample weights, whereas in adversarial training, this outer loop is devoted to the training of model parameters.

For simplicity, we adopt a widely-used approach in our paper: Performing multiple gradient ascent steps to generate adversarial examples, followed by an additional gradient ascent step over $h$. For a more thorough understanding, please see Algorithm \ref{algo:1}. If we can solve the inner maximization problem nearly optimally, then we can ensure that the sequence generated by Algorithm \ref{algo:1} converges to the global optimal solution. You can find further details in~\citet[Theorem 2]{sinha2018certifying}.

\subsection{Feature stability analysis}
\label{subsec:feature-analysis}
As an additional benefit, if we select an alternative cost function, different from the one proposed in \eqref{eq:cost_f}, our evaluation criterion $\mathfrak R(\beta,r)$ can serve as an effective metric for assessing feature stability within machine learning models. If we want to evaluate the stability of the $i$-th feature, we can modify the distance function $d$ in \eqref{eq:cost_f} as 
\begin{equation*}
    d(z,\hat z)=\|z_{(i)}-\hat z_{(i)}\|_2^2 + \infty\cdot \|z_{(,-i)}-\hat z_{(,-i)}\|_2^2,
\end{equation*}
where $z_{(i)}$ represents the $i$-th feature of $z$, while $z_{(,-i)}=z\backslash z_{(i)}$ denotes all variables in $z$ except for the $i$-th one. This implies that during evaluation, we are only permitted to perturb the $i$-th feature while keeping all other features unchanged.


Substituting $d(z,\hat z)$ in problem \eqref{equ:D-risk-kl}, we could obtain the corresponding feature stability criterion $\mathfrak R_i(\beta,r)$, which provides a quantitative stability evaluation of how robust the model is with respect to changes in the $i$-th feature.
Specifically, a higher value of $\mathfrak R_i(\beta,r)$ indicates greater stability of the corresponding feature against potential shifts.

\section{Visualizations on stylized / toy examples}
\label{sec:toy}

\begin{figure}[t]
  \centering
  \subfloat[Original Dataset]{\includegraphics[width=0.24\textwidth]{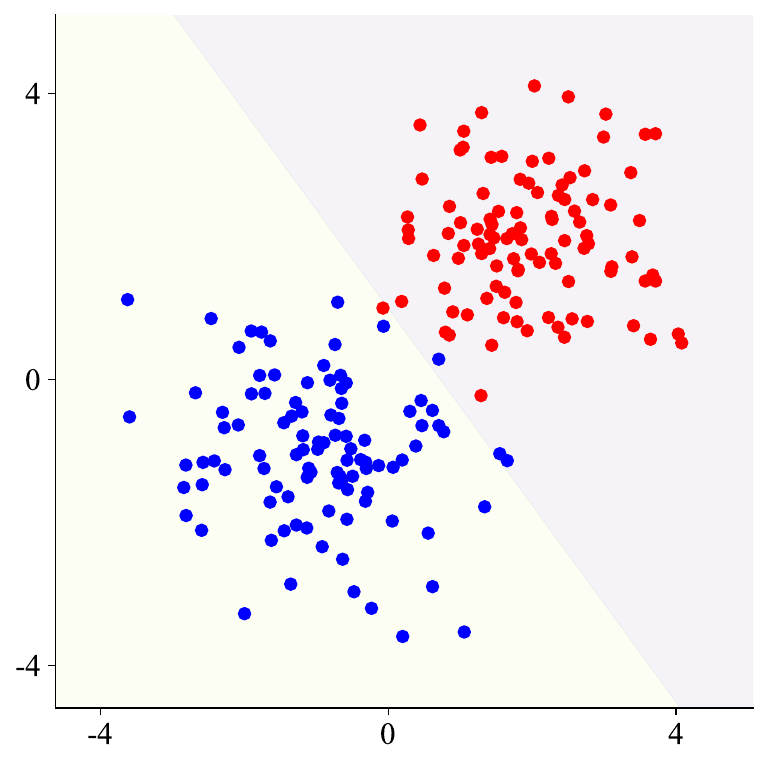}\label{fig1:overall_subfig1}}
  \subfloat[$\theta_1=+\infty,\theta_2=0.2$]{\includegraphics[width=0.24\textwidth]{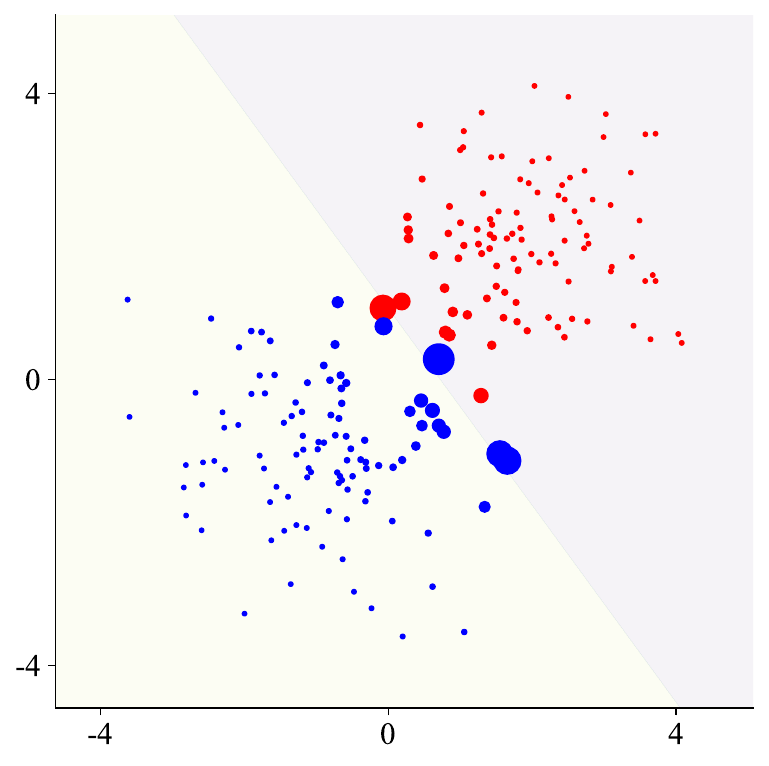}\label{fig1:overall_subfig3}}
  \subfloat[$\theta_1=0.2,\theta_2=+\infty$]{\includegraphics[width=0.24\textwidth]{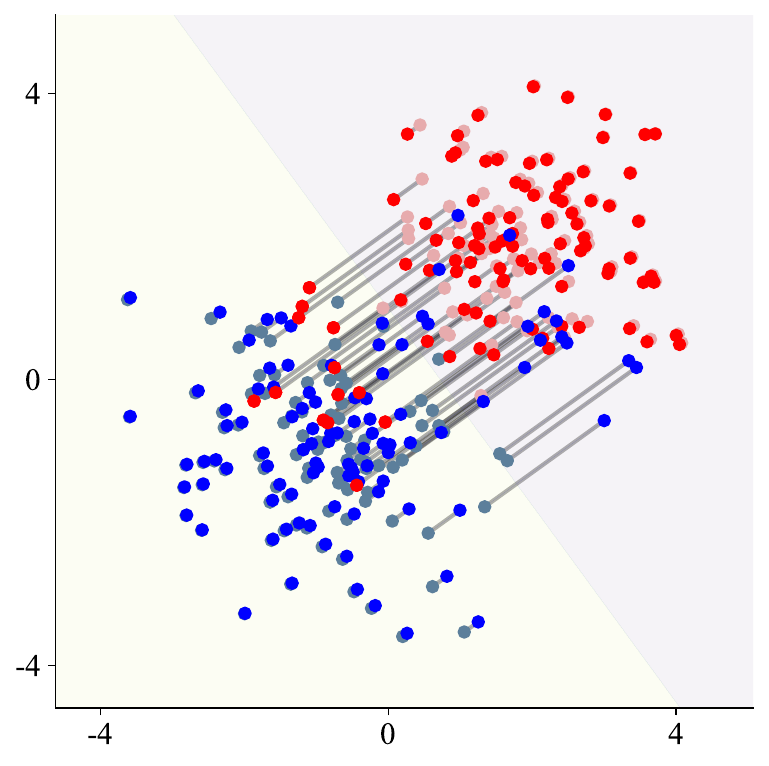}\label{fig1:overall_subfig2}}
  \subfloat[$\theta_1=\theta_2=0.4$]{\includegraphics[width=0.24\textwidth]{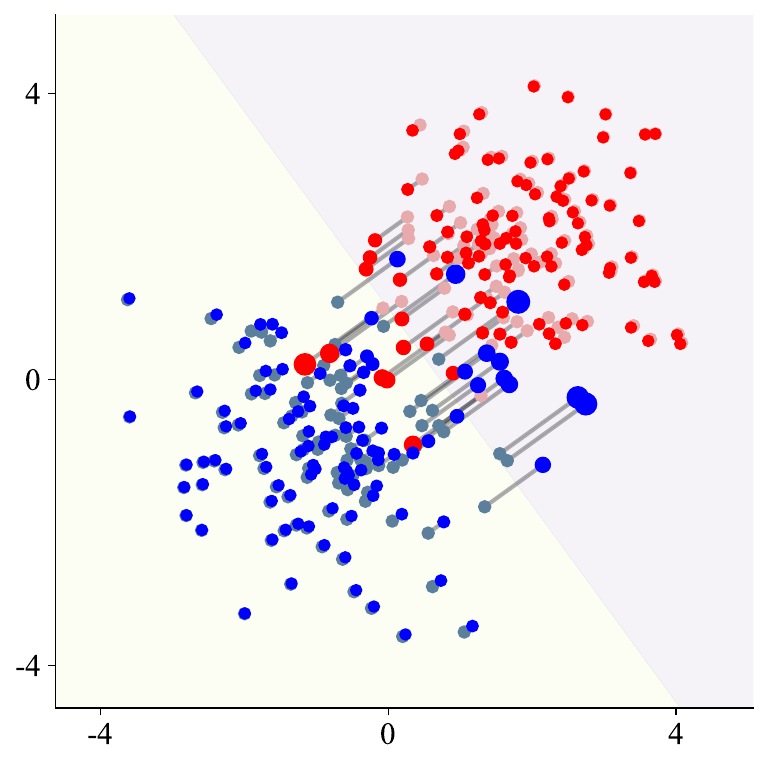}\label{fig1:overall_subfig4}}
\caption{Visualizations of the original dataset and the most sensitive distribution $\Q^\star$ produced by cross-entropy loss function under different $\theta_1,\theta_2$. The original prediction error is $0.1$, and the risk threshold is $0.5$.}
 \label{fig:overall}  
\end{figure}

In this section, we use a toy example to visualize the most sensitive distribution $\Q^\star$ based on Remark \ref{remark:q-star}, which provides intuitive insights into the structure of $\Q^\star$. 

We consider a two-dimensional binary classification problem.
We generate 100 samples for $Y=0$ from distribution $\mathcal N([2,2]^T, I_2)$, and 100 samples for $Y=1$ from distribution $\mathcal N([-1,-1]^T, I_2)$.
The model $f_\beta(\cdot)$ under evaluation is logistic regression (LR).
In this section, we choose $\phi(t)=t\log t - t +1$.
To explore the effects of varying the adjustment parameters, we fix $1/\theta_1+1/\theta_2=5$. We use the cross-entropy loss function, set the risk threshold to be 0.5 (the original loss was 0.1), and solve the problem \eqref{equ:D-risk-kl}.
In Figure \ref{fig1:overall_subfig3}-\ref{fig1:overall_subfig4}, we visualize the most sensitive distribution $\Q^\star$ in each setting, where the decision boundary of $f_\beta(\cdot)$ is indicated by the boundary line, colored points represent the perturbed samples, shadow points represent the original samples, and the size of each point is proportional to its sample weight in $\Q^\star$.
Corresponding with the analysis in Section \ref{subsec:framework}, we have the following observations: 

\begin{enumerate}[label=(\roman*),leftmargin=*]
    \item When $\theta_1=+\infty$, our stability criterion only considers sub-population shifts. From Figure \ref{fig1:overall_subfig3}, we notice a significant increase in weight assigned to a limited number of samples near the boundary. This aligns with the works of \citet{namkoong2022minimax,gupta2023the}, which emphasize tail performance analysis.
    \item When $\theta_2 = +\infty$, the stability criterion only considers data corruptions. From Figure \ref{fig1:overall_subfig2}, a significant number of samples are severely perturbed to adhere to the predefined risk threshold.
    \item When $\theta_1=\theta_2=0.4$, in Figure \ref{fig1:overall_subfig4}, a more balanced $\Q^\star$ is observed, reflecting the incorporation of both data corruptions and sub-population shifts. 
    This showcases a scenario where samples undergo moderate and reasonable perturbations, and the sensitive distribution is not disproportionately concentrated on a limited number of samples. Such a distribution is a more holistic and reasonable approach to evaluating stability in practice, taking into account a broader range of potential shifts.
\end{enumerate}

Furthermore, we showcase the most sensitive distributions with 0/1 loss.
We set the error rate threshold $r$ to be 30\%.
The results are shown in Figure \ref{fig:overall2}.
From the results, we have the following observations:
\begin{enumerate}[label=(\roman*),leftmargin=*]
    \item Similar to the phenomenon above, when $\theta_2=+\infty$, the stability criterion only considers data corruptions; and when $\theta_1=+\infty$, it only considers sub-population shifts.
    \item Different from Figure \ref{fig:overall}, since we use 0/1 loss here, the perturbed samples are all near the boundary.
\end{enumerate}

\begin{figure}[t]
  \centering
  \subfloat[Original Dataset]{\includegraphics[width=0.24\textwidth]{original.pdf}\label{fig:2overall_subfig1}}
  \hfill
  \subfloat[$\theta_1=1.0,\theta_2=0.25$]{\includegraphics[width=0.24\textwidth]{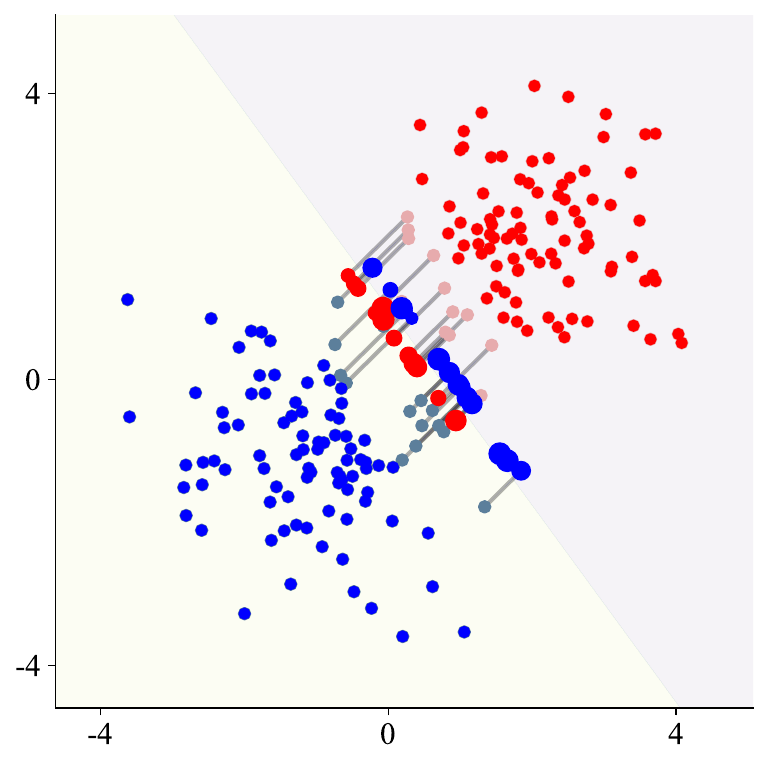}\label{fig:2overall_subfig3}}
  \hfill
  \subfloat[$\theta_1=0.2,\theta_2=+\infty$]{\includegraphics[width=0.24\textwidth]{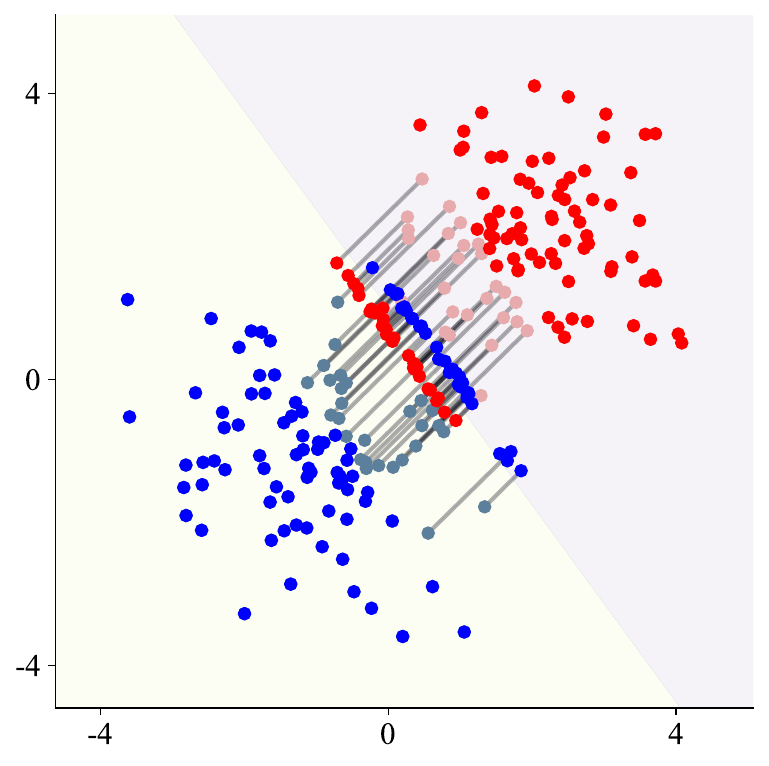}\label{2fig:overall_subfig4}}
  \hfill
  \subfloat[$\theta_1=+\infty,\theta_2=0.2$]{\includegraphics[width=0.24\textwidth]{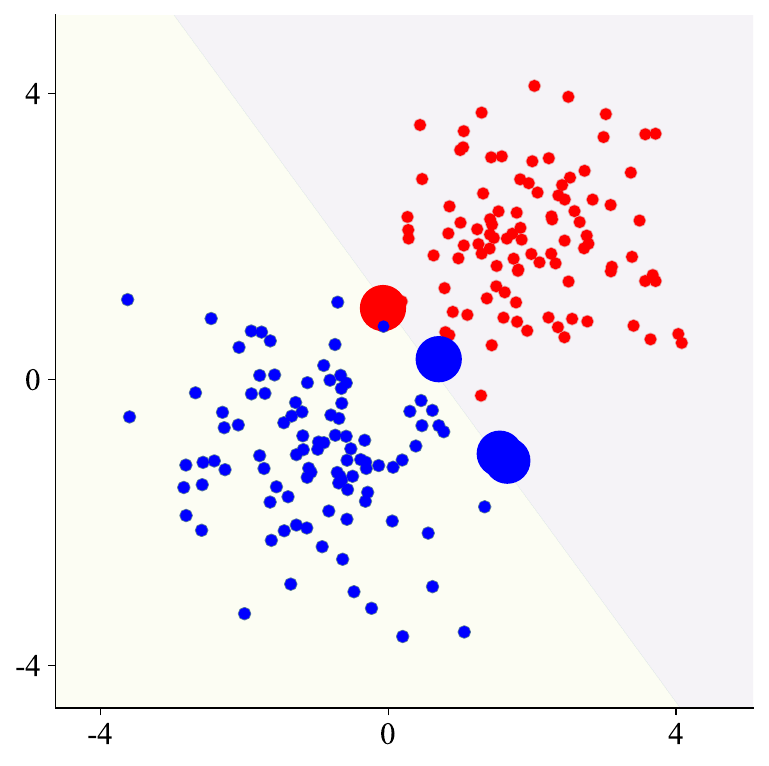}\label{2fig:overall_subfig5}}
\caption{Visualizations of the original dataset and the most sensitive distribution $\Q^\star$ with \emph{0/1} loss function under different $\theta_1,\theta_2$. The original prediction error rate is $1\%$, and the error rate threshold $r$ is set to $30\%$.}
 \label{fig:overall2}  
\end{figure}

\begin{figure}[t]
  \centering
  \subfloat[Error rate $r=20\%$]{\includegraphics[width=0.24\textwidth]{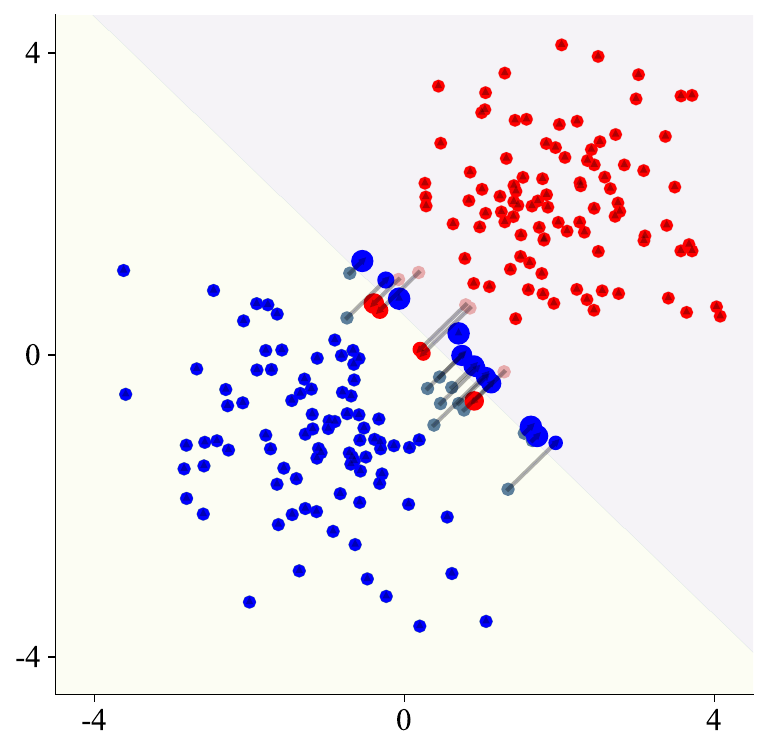}}
  \hfill
  \subfloat[Error rate $r=40\%$]{\includegraphics[width=0.24\textwidth]{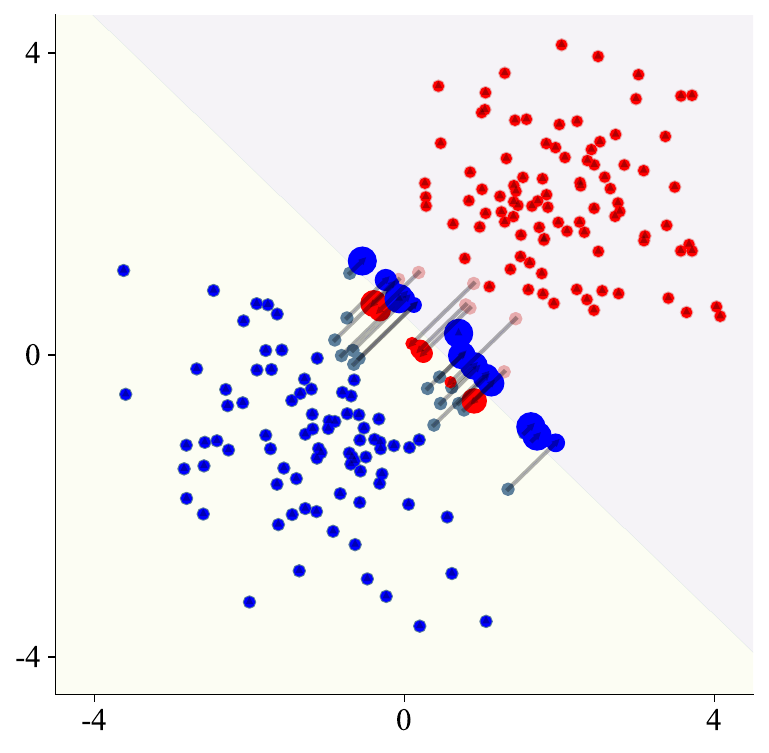}}
  \hfill
  \subfloat[Error rate $r=60\%$]{\includegraphics[width=0.24\textwidth]{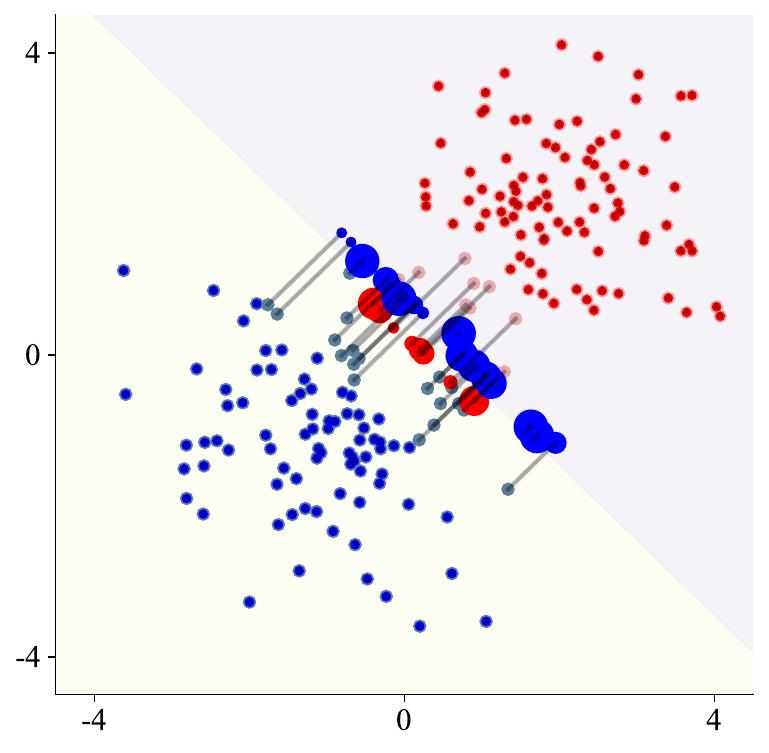}}
  \hfill
  \subfloat[Error rate $r=80\%$]{\includegraphics[width=0.24\textwidth]{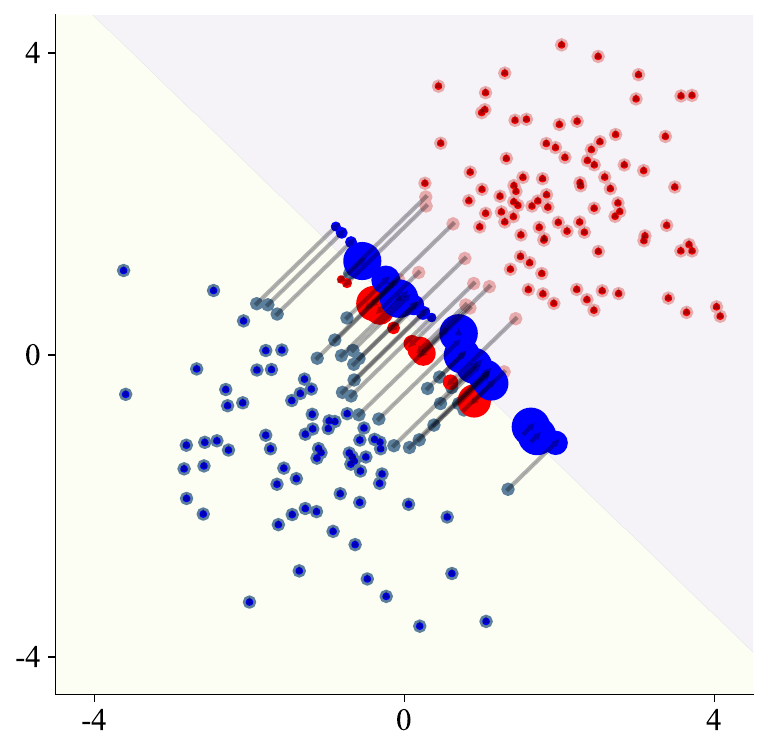}}
\caption{Visualizations of the most sensitive distribution $\Q^\star$ with \emph{0/1} loss function under different error rate threshold. We set $\theta_1=1.0$ and $\theta_2=0.25$ here.}
 \label{fig:toy-varying-r}  
\end{figure}

For fixed $\theta_1$ and $\theta_2$, we vary the error rate threshold $r$ and visualize the most sensitive distribution $\Q^\star$ in Figure \ref{fig:toy-varying-r}.
We set $\theta_1=1.0$ and $\theta_2=0.25$, and our stability criterion will consider both data corruptions and sub-population shifts.

Finally, in Figure \ref{fig:toy-convergence}, we plot the curve of $\mathbb{E}_{\Q^{(t)}}[W\cdot \ell(\beta,Z)]$ with respect to the epoch number $t$.
From the results, it's evident that the infeasibility error of the sequence generated by our algorithm tends towards zero. This implies that the final expectation over the most sensitive distribution $\Q^{(T)}$ will match the predefined threshold $r$.

\begin{figure}[htbp]
  \centering
  \subfloat[General nonlinear loss function.]{\includegraphics[width=0.35\textwidth]{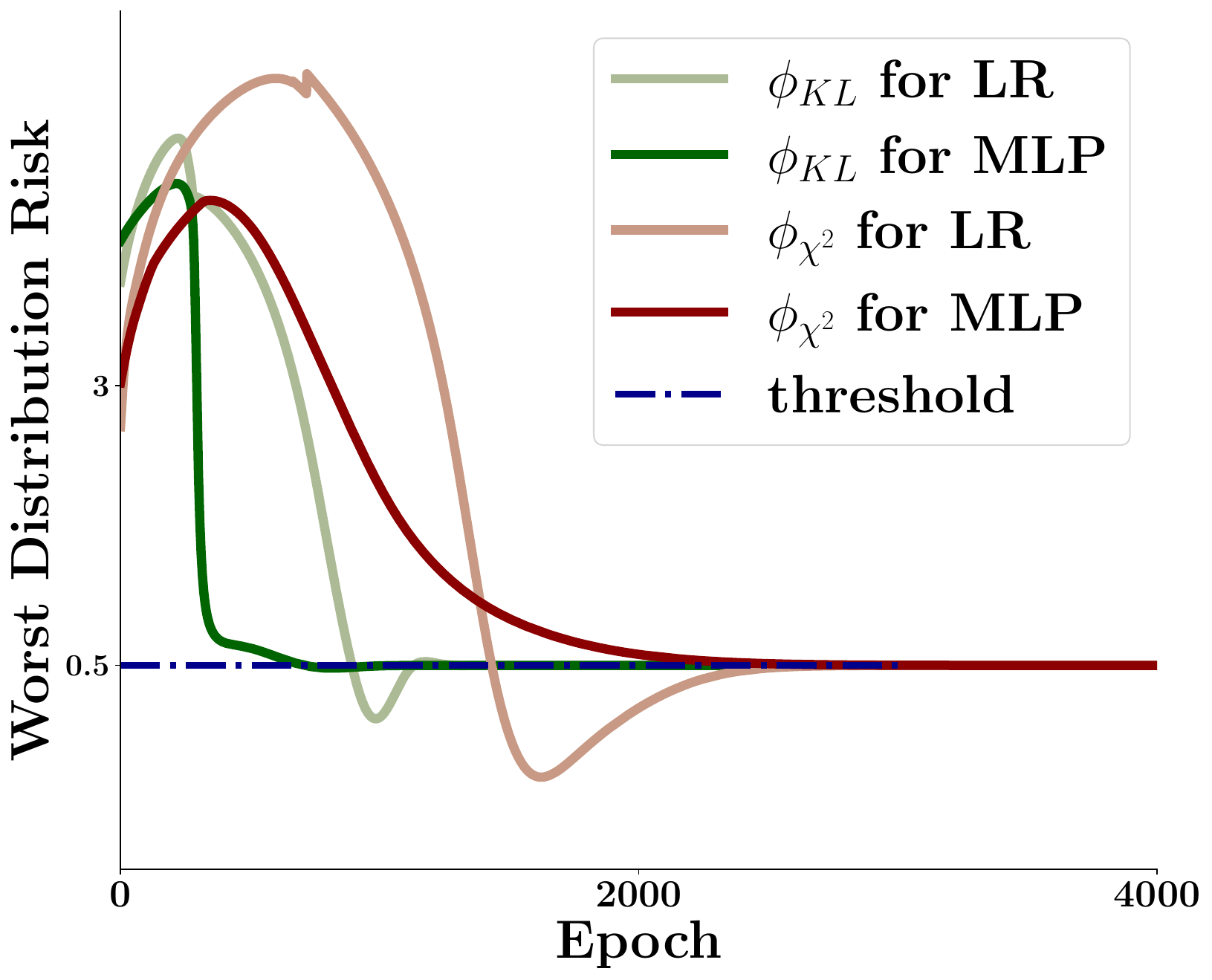}}
  \hspace{2em}
  \subfloat[0/1 loss function.]{\includegraphics[width=0.35\textwidth]{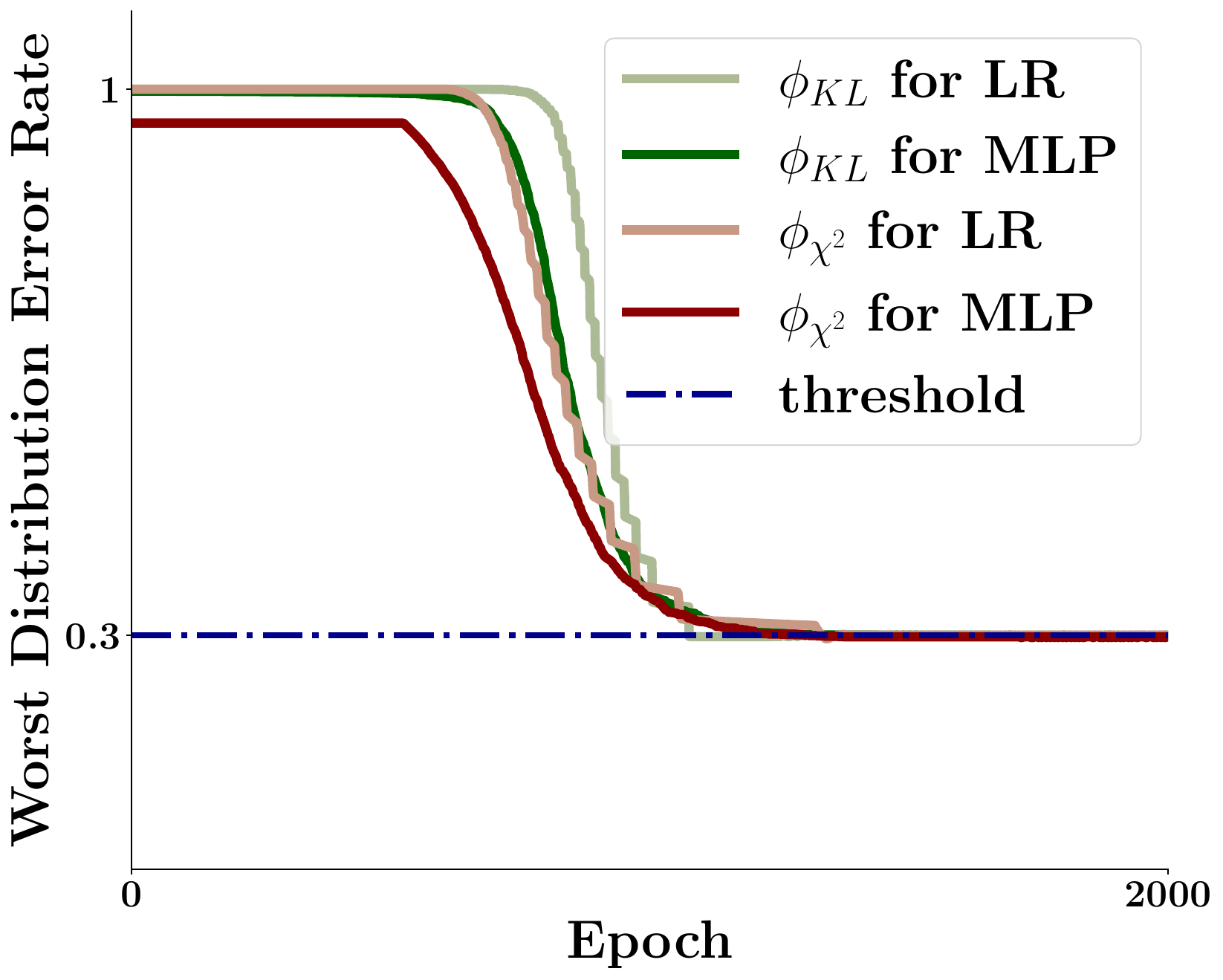}}
\caption{The convergence of $\mathbb{E}_{\Q^{(t)}}[W\cdot \ell(\beta,Z)]$ w.r.t. epoch $t$. (a): Use general nonlinear loss function (cross-entropy loss) with $r=0.5$. (b): Use 0/1 loss function with $r=30\%$. Here $\phi_{\text{KL}}$ denotes $\phi(t)=t\log t - t+1$, and $\phi_{\chi^2}$ denotes $\phi(t)=(t-1)^2$.}
 \label{fig:toy-convergence}  
\end{figure}

\section{Experiments}
\label{sec:exp}

In this section, we explore real-world applications to show the practical effectiveness of our stability evaluation criterion, including how this criterion can be utilized to compare the stability of both models and features, and to inform strategies for further enhancements.\\

\noindent \textbf{Datasets.}\quad We use three real-world datasets, including ACS Income dataset, ACS Public Coverage dataset, and COVID-19 dataset.
\begin{itemize}[leftmargin=*]
    \item \textbf{ACS Income dataset}. The dataset is based on the American Community Survey (ACS) Public Use Microdata Sample (PUMS)~\citep{ding2022retiring}. The task is to predict whether an individual’s income is above \$50,000. We filter the dataset to only include individuals above the age of 16, usual working hours of at least 1 hour per week in the past year, and an income of at least \$100. The dataset contains individuals from all American states, and we focus on California (CA) in our experiments. We follow the data pre-processing procedures in~\citet{liu2021towards}. The dataset comprises a total of 76 features, with the majority of categorical features being one-hot encoded to facilitate analysis. In our experiments, we sample 2,000 data points from CA for model training, and another 2,000 for evaluation. When involving algorithmic interventions in Section \ref{subsec:feature}, we further sample 5,000 points to compare the performances of different algorithms.
    \item \textbf{ACS Public Coverage dataset}. The dataset is also based on  ACS PUMS~\citep{ding2022retiring}. The task is to predict whether an individual has public health insurance. We focus on low-income individuals who are not eligible for Medicare by filtering the dataset to only include individuals under the age of 65 and with an income of less than \$30,000.
    Similar to the ACS Income dataset,  we focus on individuals from CA in our experiments. We follow the data pre-processing procedures in~\citet{liu2021towards}. The dataset comprises a total of 42 features, with the majority of categorical features being one-hot encoded to facilitate analysis. In our experiments, we sample 2,000 data points from CA for model training, and another 2,000 for evaluation. When involving algorithmic interventions in Section \ref{subsec:feature}, we further sample 5,000 points to compare the performances of different algorithms.
    \item \textbf{COVID-19 dataset}. The COVID-19 dataset contains COVID patients from Brazil, which is based on SIVEP-Gripe data~\citep{baqui2020ethnic}. It has 6882 patients from Brazil recorded between Februrary 27-May 4, 2020. There are 29 features in total, including comorbidities, symptoms, and demographic characteristics. The task is to predict the mortality of a patient, which is a binary classification problem.
    In our experiments, we split the dataset with a ratio of 1:1 for training and evaluation sets.
\end{itemize}

\noindent Throughout the experiments, we set $1/\theta_1+1/\theta_2=5$ for adjustment parameters $\theta_1$ and $\theta_2$.\\

\noindent\textbf{Algorithms under evaluation}\quad 
Before presenting experimental results, we will initially introduce the formulations of various algorithms used to evaluate the effectiveness of their interventions.
In Section \ref{subsec:model}, we evaluate Adversarial Training (AT) ~\citet{sinha2018certifying} and Tilted ERM~\citep{li2023tilted}. 
In Section \ref{subsec:feature}, we introduce the Targeted AT. Here are their mathematical formulations: 
\begin{enumerate}[label=(\roman*),leftmargin=*]
    \item AT: 
    \begin{equation}
        \min_{\beta}\bigg\{\mathbb{E}_{\P_0}[\phi_{\gamma}(\beta,Z)] := \mathbb{E}_{\P_0}\left[\sup_{z\in\Zscr} \ell(\beta,Z)-\gamma c(Z,\hat Z)\right]\bigg\},
    \end{equation}
    where $c(z,\hat z)=\|x-\hat x\|_2^2 + \infty\cdot |y-\hat y|$, and $\gamma$ is the penalty hyper-parameter.
    \item Tilted ERM:
    \begin{equation}
        \min_{\beta} t\log\bigg(\mathbb{E}_{\P_0}\left[\exp\left(\frac{\ell(\beta,Z)}{t}\right)\right]\bigg),
    \end{equation}
    where $t$ is the temperature hyper-parameter.
    \item Targeted AT: 
    \begin{equation}
    \min_{\beta}\bigg\{\mathbb{E}_{\P_0}[\phi_{\gamma}(\beta,Z)] = \mathbb{E}_{\P_0}\left[\sup_{z\in\Zscr} \ell(\beta,Z)-\gamma c(Z,\hat Z)\right]\bigg\}.
    \end{equation}
    In this case, $c(z,\hat z)=\|z_{(i)}-\hat z_{(i)}\|_2^2 + \infty\cdot \|z_{(,-i)}-\hat z_{(,-i)}\|_2^2$, where $z_{(i)}$ denotes the target feature of $z$, $z_{(,-i)}$ denotes all the other features and 
    $\gamma$ is the penalty hyper-parameter.
    By choosing this $c(z,\hat z)$, the targeted AT will only perturb the target feature while keeping the others unchanged.
\end{enumerate}

\noindent\textbf{Training Details.}\quad 
In our experiments, we use LR for linear model and a two-layer MLP for neural network.
We use PyTorch Library~\citep{paszke2019pytorch} throughout our experiments. 
The number of hidden units of MLP is set to 16. 
As for the models under evaluation in Section \ref{sec:exp}, \textrm{(i)} for AT, we vary the penalty parameter $\gamma\in\{0.1,0.2,\dots,1.0\}$ and select the best $\gamma$ according to the validation accuracy. The inner number of inner optimization iterates is set to 20;
\textrm{(ii)} for Tilted ERM, we vary the temperature parameter $t\in\{0.1,0.2,\dots,1.0\}$ and select the best $t$ according to the validation accuracy.
Throughout all experiments, the ADAM optimizer with a learning rate of $1e^{-3}$ is used.
All experiments are performed using a single NVIDIA GeForce RTX 3090.

\subsection{Model stability analysis}
\label{subsec:model}
In this section, we first provide more in-depth empirical analyses of our proposed criterion, and demonstrate how it can reflect a model's stability with respect to data corruptions and sub-population shifts.
We focus on the income prediction task for individuals from CA, using the ACS Income dataset.

\vspace{2mm}
\noindent\textbf{Excess risk decomposition.}\quad 
Recall that our stability evaluation misleads the model to a pre-defined risk threshold by perturbing the original distribution $\P_0$ in two ways, i.e. data corruptions and sub-population shifts.
Based on the optimal solutions  $\Q^\star\in\mathcal{P}(\Zscr\times \Wscr)$ of problem \eqref{eq:primal}, we can compute the excess risk $\Delta=\mathbb{E}_{\Q^\star}[W\cdot \ell(\beta,Z)]-\mathbb{E}_{\P_0}[\ell(\beta,Z)]$ into two parts satisfying $\Delta=\Delta_{\mathrm{I}}+\Delta_{\mathrm{II}}$:

\begin{equation}
\label{equ:decomposition}
\begin{aligned}
    \Delta_{\mathrm{I}} &\coloneqq \mathbb{E}_{\Q^\star_{Z}}[\ell(\beta,Z)]-\mathbb{E}_{\P_0}[\ell(\beta,Z)],\\
    \Delta_{\mathrm{II}} &\coloneqq \mathbb{E}_{\Q^\star}[W\cdot\ell(\beta,Z)]-\mathbb{E}_{\Q^\star_Z}[\ell(\beta,Z)],
\end{aligned}
\end{equation}
 where $\Q^\star_{Z}$ is the marginal distribution of $\Q^\star$ w.r.t $Z$. Thus, from this decomposition, we can see that $\Delta_{\mathrm{I}}$ denotes the excess risk induced by data corruptions (data samples $\hat z\rightarrow z$), and $\Delta_{\mathrm{II}}$ denotes that induced by sub-population shifts (probability density $1\rightarrow w$).
In this experiment, for a MLP model trained with empirical risk minimization (ERM), we use the cross-entropy loss and set the risk threshold to be 3.0.
In Figure \ref{fig:income_decomposition}, we vary the $\theta_1$ and $\theta_2$ and plot the $\Delta_{\mathrm I},\Delta_{\mathrm{II}}$ in each setting.
The results align with our theoretical understanding that a decrease in $\theta_1$ leads our evaluation method to place greater emphasis on data corruptions. 
Conversely, a reduction in $\theta_2$ shifts the focus of our evaluation towards sub-population shifts. 
This observation confirms the adaptability of our approach in weighing different types of distribution shifts based on the values of $\theta_1$ and $\theta_2$.

\begin{figure*}[h]
  \centering
\subfloat[Excess risk decomposition.]{\includegraphics[width=0.32\textwidth, height=0.21\textwidth]{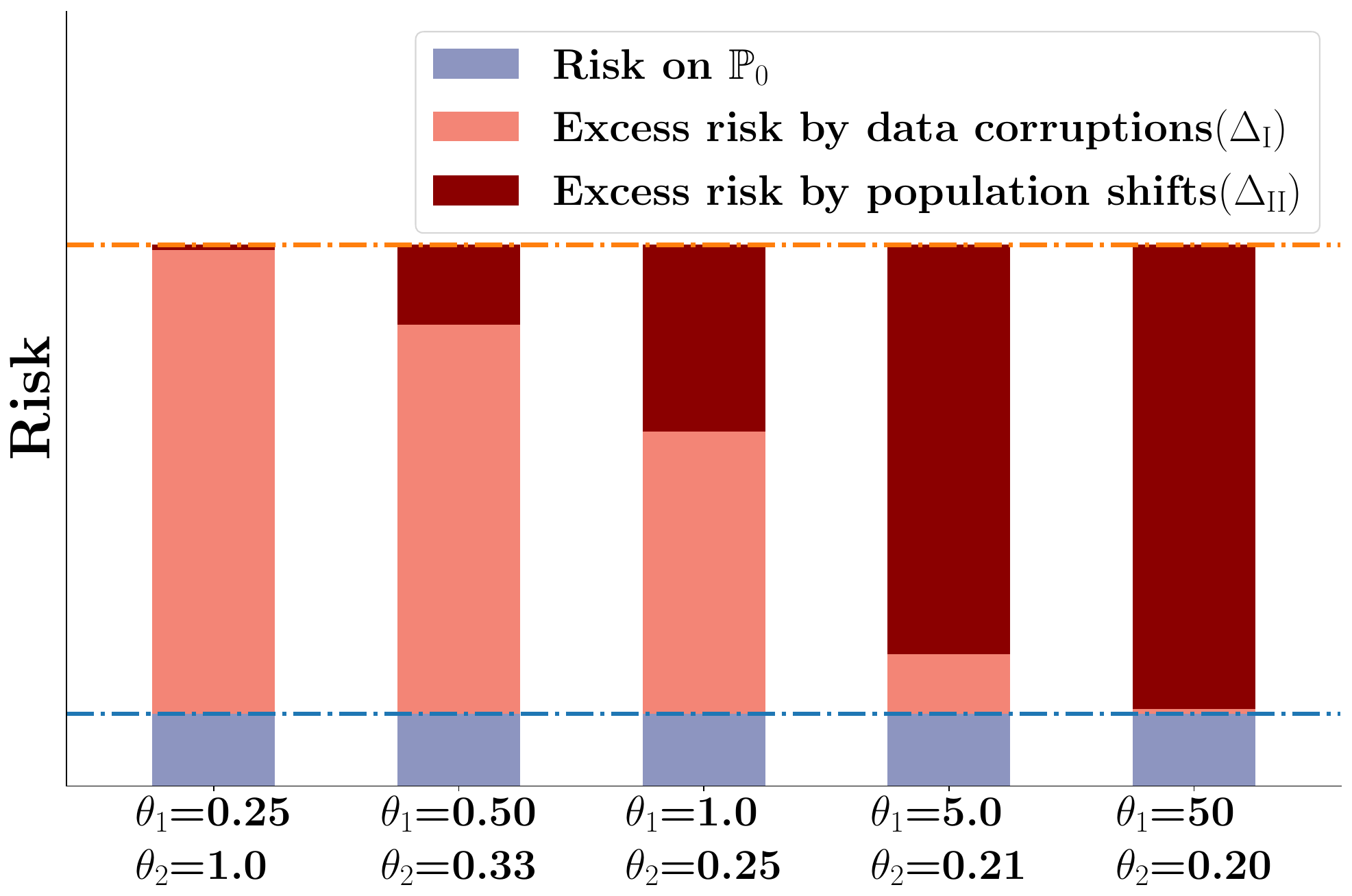}\label{fig:income_decomposition}}
\subfloat[Convergence of $\mathbb{E}_{\Q^{(t)}}[W\cdot \ell(\beta,Z)\text{]}$ w.r.t. epoch $t$.]{\includegraphics[width=0.32\textwidth, height=0.21\textwidth]{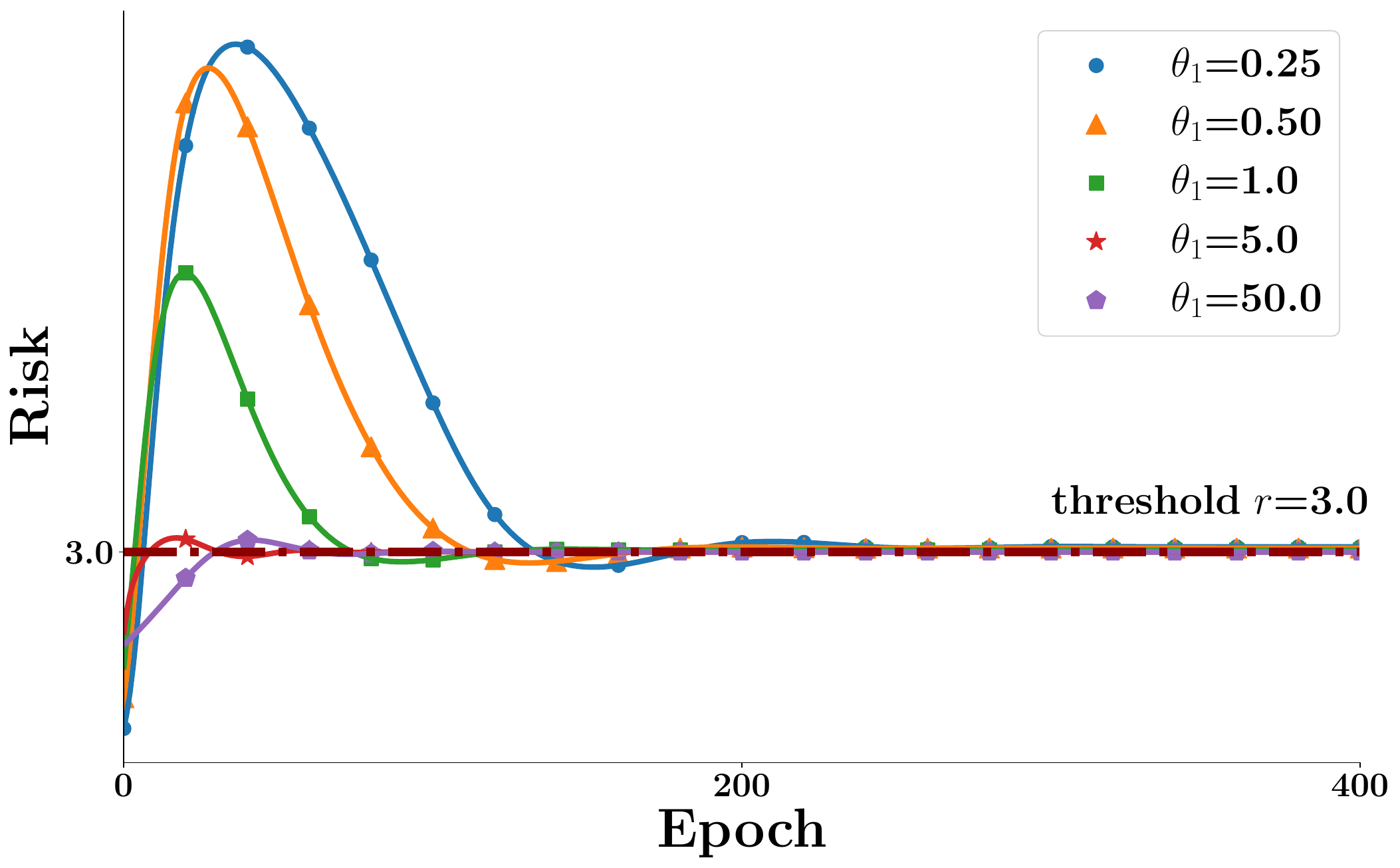}\label{fig:income_convergence}}
\subfloat[Stability measure.]{\includegraphics[width=0.32\textwidth, height=0.21\textwidth]{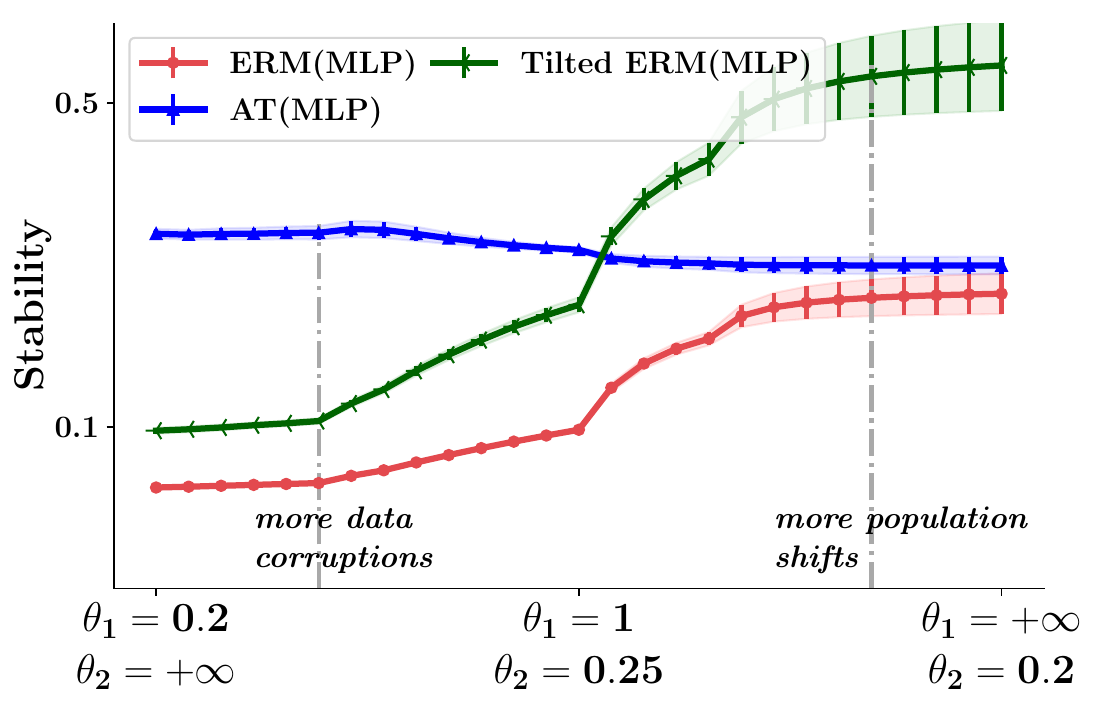}\label{fig:income_general_mlp}}

  \caption{Results of the income prediction task. (a): The excess risk decomposition under different values of $\theta_1$ and $\theta_2$ according to~\eqref{equ:decomposition}.
  (b): The curve of the risk on the most sensitive distribution $\Q^\star$ during optimization for different choices of $\theta_1$ and $\theta_2$, which converge to the pre-defined risk threshold. 
  The models under evaluation in (a) and (b) are both ERM (MLP).
  (c): The stability measure for MLP models trained with ERM, AT, and Tilted ERM, under varying $\theta_1$ and $\theta_2$. Error bars denote the standard deviations over multiple runs.}
 \label{fig:income_prediction}  
\end{figure*}
\noindent\textbf{Convergence of our optimization algorithm.}\quad 
In Figure~\ref{fig:income_convergence}, we plot the curve of the risk on $\Q^{(t)}$ w.r.t. the epoch number $t$ throughout the optimization process.
For different values of $\theta_1$ and $\theta_2$, we observe that the risk consistently converges to the pre-defined risk threshold of $r=3.0$.
This empirical observation is in agreement with our theoretical investigation, demonstrating the reliability and effectiveness of our optimization approach.

\vspace{2mm}
\noindent\textbf{Reflection of stability.}\quad 
We then proceed to compare the stability of MLP models trained with three well-established methods, including ERM, AT, and Tilted ERM.
AT is specifically designed to enhance the model's resilience to data corruptions, whereas Tilted ERM, through its use of the log-sum-exp loss function, aims to prioritize samples with elevated risks, potentially enhancing stability in the presence of sub-population shifts.
For our analysis, we set the risk threshold $r$ to 3.0, vary $\theta_1$ and $\theta_2$, and plot the resulting stability measure $\mathfrak R(\beta,3.0)$ for each method.

From Figure \ref{fig:income_general_mlp}, we have the following observations:
(i) Both robust learning methods exhibit markedly higher stability compared to ERM;
(ii) AT exhibits greater stability in the context of data corruptions, while Tilted ERM shows superior performance in scenarios involving sub-population shifts. 
These findings align with our initial hypotheses regarding the strengths of these methods;
(iii) Furthermore, the results suggest that robust learning methods tailored to specific types of distribution shifts \emph{may face challenges in generalizing to other contexts}. 
Therefore, accurately identifying the types of shifts to which a model is most sensitive is crucial in practice, as it can inform machine learning engineers on strategies to further refine and improve the model's robustness and efficacy.
This insight underscores the significance of our proposed stability evaluation framework. 
It offers a comprehensive and unified perspective on a model's stability across various types of distribution shifts, enabling a more holistic understanding and strategic approach to enhancing model robustness and reliability.

Furthermore, the results of models' stability on the ACS PubCov dataset and the COVID-19 dataset are shown in Figure \ref{fig:more-model-analysis}.
We can observe similar phenomenon as the ACS Income dataset: 
\begin{figure}[h]
  \centering
  \subfloat[PubCov dataset.]{\includegraphics[width=0.45\textwidth]{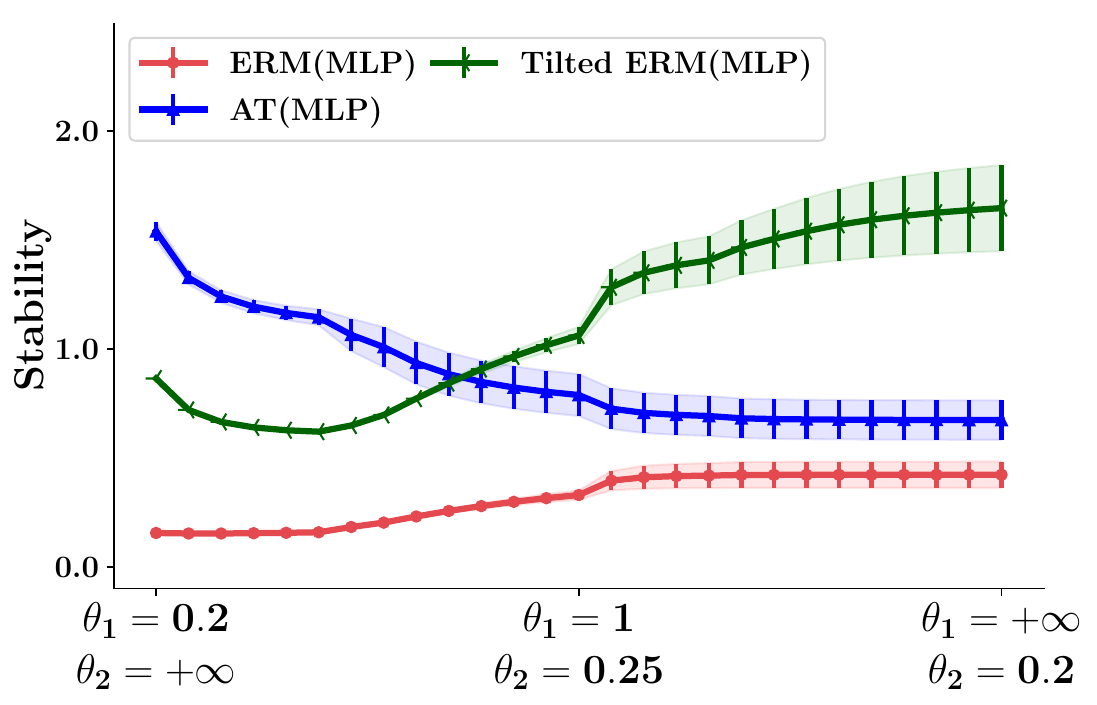}}
  \subfloat[COVID-19 dataset.]{\includegraphics[width=0.45\textwidth]{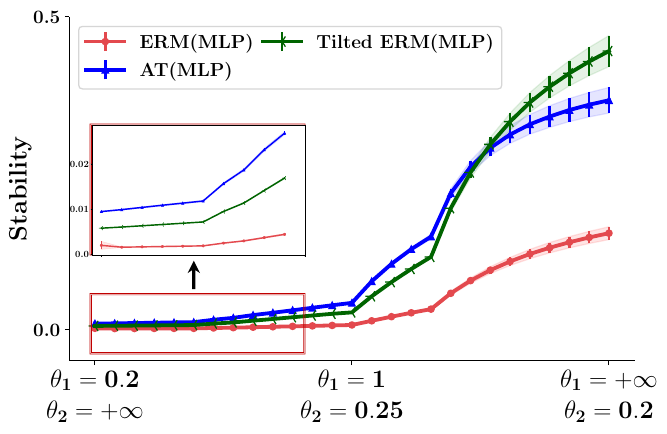}}
  \caption{The stability measure for MLP models trained with ERM, AT, and Tilted ERM on ACS PubCov dataset and COVID-19 dataset.}
\label{fig:more-model-analysis}  
\end{figure}
\vspace{-2mm}
\begin{enumerate}[label=(\roman*),leftmargin=*]
    \item When $\theta_1$ is small, our stability measure pays more attention to data corruptions. Therefore, AT performs better than Tilted ERM and ERM.
    \item When $\theta_2$ is small, the main focus shifts to population shifts, where Tilted ERM is more preferred.
\end{enumerate}

Besides, it is noteworthy that the standard deviation of the stability measure estimation increases as $\theta_1$ approaches infinitely (we set it to 100 in our experiments).
When fixing the evaluation data,  the standard deviations—indicating the randomness inherent to our computational algorithm—are relatively small.
This observation points to the randomness of sampling as the primary factor.
Furthermore, the introduction of \(\theta_1=+\infty\) brings a statistical cost in calculating the stability measure, as demonstrated in~\citet{namkoong2022minimax}.

\subsection{Feature stability analysis}
\label{subsec:feature}

Building upon our previous findings, we further investigate the applicability of feature stability analysis across multiple prediction tasks, including income, insurance, and COVID-19 mortality prediction. 
By examining feature stability, we gain valuable insights into the specific attributes that significantly influence model performance.
It provides a  principle approach to enhance our understanding of the risky factors contributing to overall model instability, and thereby helps to \emph{discover potential discriminations and improve model robustness and fairness}.
Throughout all the experiments, we use 0/1 loss function and set the error rate threshold $r$ to be $40\%$. 
The adjustment parameter $\theta_1$ is set to 1.0, and $\theta_2$ is 0.25.\\
\begin{figure*}[htbp]
  \centering
\subfloat[\emph{Income} Prediction: Feature Stability]{\includegraphics[width=0.45\textwidth]{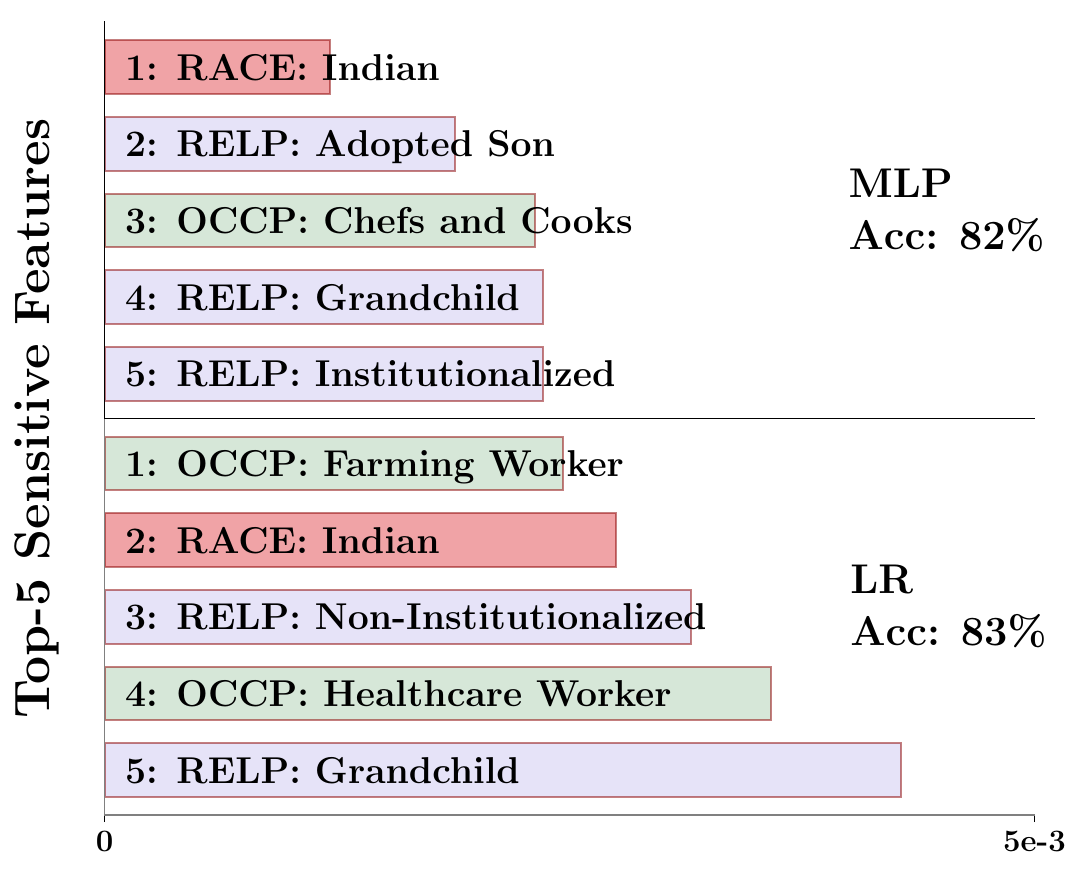}\label{fig:income_feature}}
    \hfill
  \subfloat[\emph{Income} Prediction: Worst Group Acc]{\includegraphics[width=0.45\textwidth]{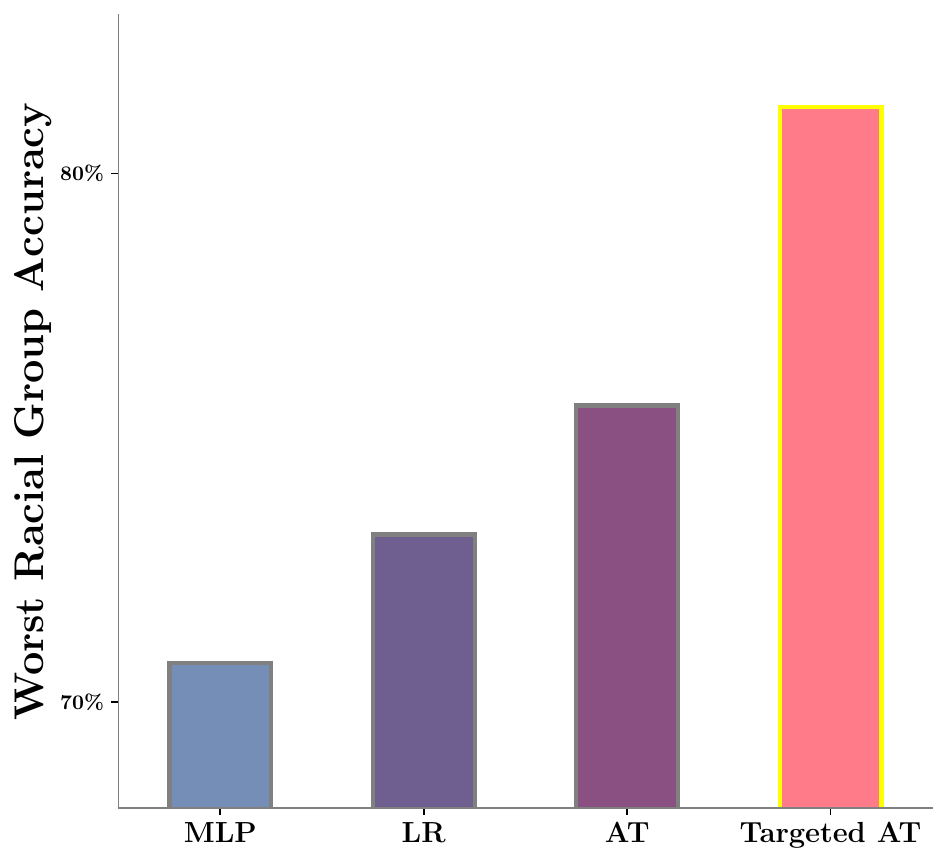}\label{fig:income_algorithmic_intervention}}
  \hfill
  \subfloat[\emph{PubCov} Prediction: Feature Stability]{\includegraphics[width=0.45\textwidth]{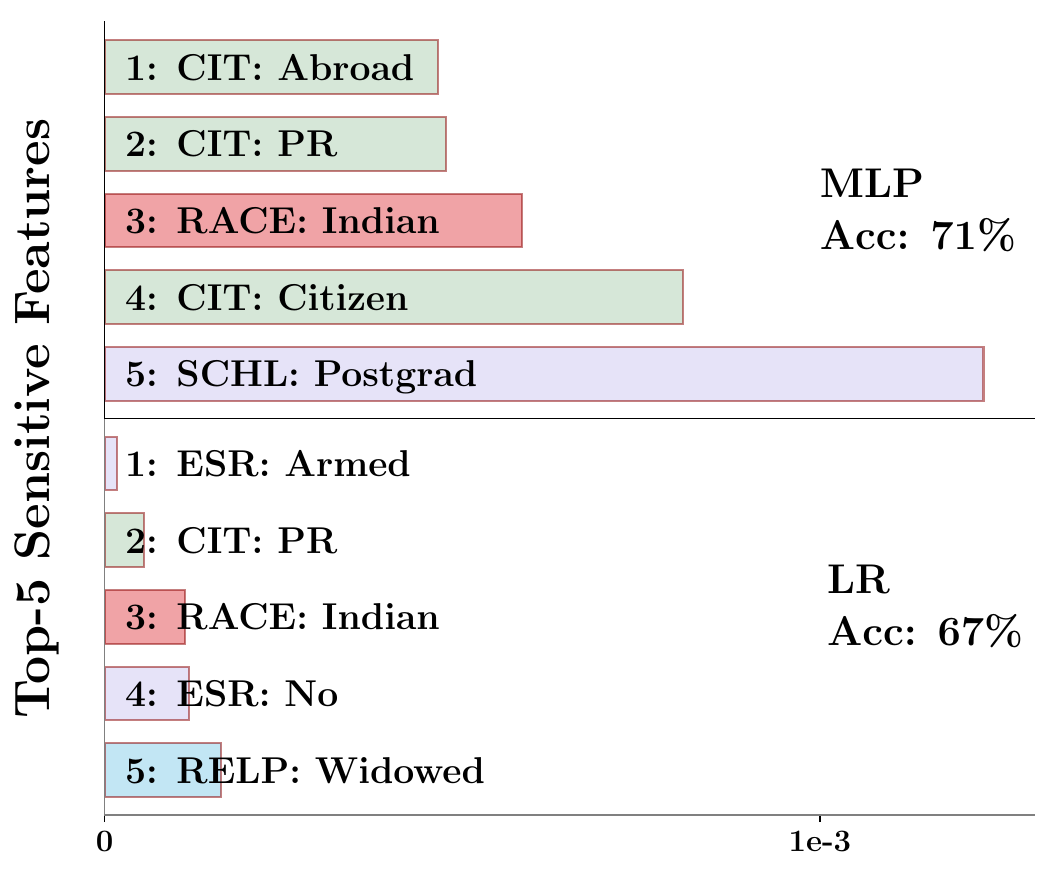}\label{fig:pubcov_feature}}
    \hfill
  \subfloat[\emph{PubCov} Prediction: Worst Group Acc]{\includegraphics[width=0.45\textwidth]{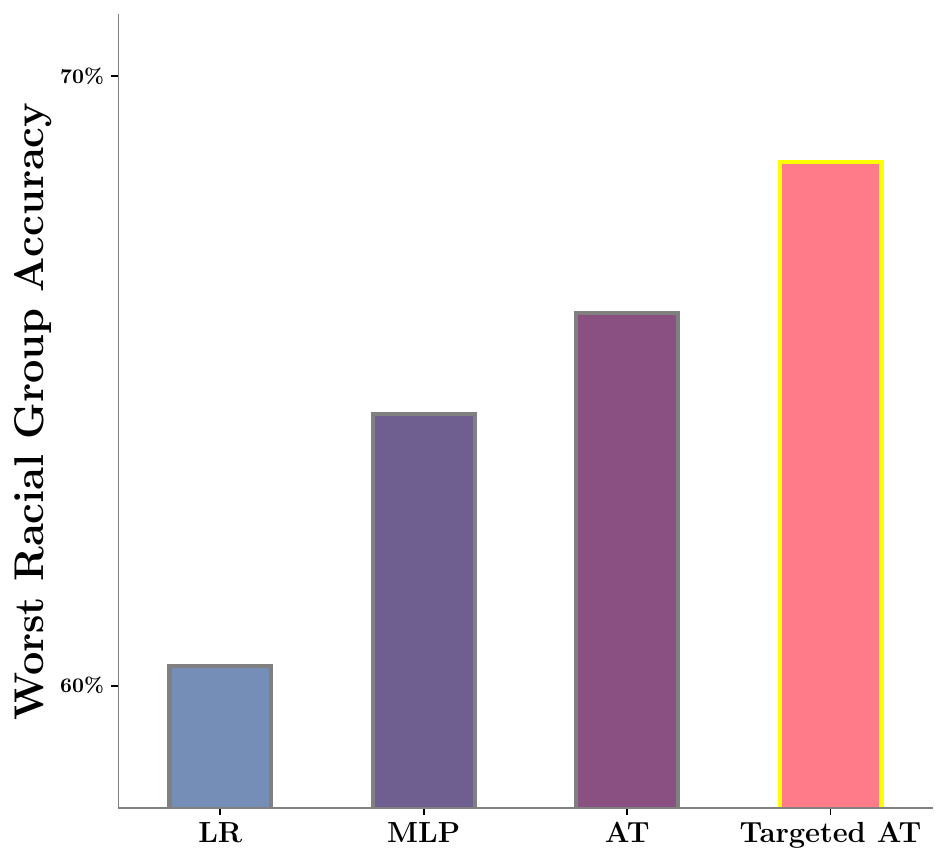}\label{fig:pubcov_algorithmic_intervention}}
\caption{Feature sensitivity analysis for income prediction and public coverage prediction. Figure (a) and (c): the top-5 sensitive feature scores for MLP and LR in the income prediction and the public coverage (PubCov) prediction tasks, where a smaller score means the corresponding feature is more sensitive. Figure (b) and (d): the worst racial group accuracy for MLP, LR, AT, and targeted AT in the income prediction and the public prediction tasks.}
 \label{fig:overall_feature}  
\end{figure*}

\noindent\textbf{Income prediction}

\vspace{2mm}
\noindent We sample 2,000 data points from ACS Income dataset for training, an additional 2,000 points for the evaluation set, and a further 5,000 points to test the effectiveness of algorithmic interventions.
For both the LR model and the MLP model, trained using ERM, we use the evaluation set to compute the feature sensitivity measure $\mathfrak{R}_i(\beta,r)$ for each feature as outlined in Section \ref{subsec:feature-analysis}.
The top-5 most sensitive features for each model – MLP and LR – are displayed in Figure \ref{fig:income_feature}. In these visualizations, distinct colors are assigned to different types of features for clarity; for example, red is used to denote racial features, while green indicates occupation features. 
From the results, we observe that:
\textrm{(i)} When the performances are similar (82\% v.s. 83\%), the LR model is less sensitive to input features, compared with the MLP model, which corresponds with the well-known Occam’s Razor.
\textrm{(ii)} Interestingly, our stability criterion reveals that both the MLP and LR models exhibit a notable sensitivity to the racial feature ``American Indian". 
This raises concerns regarding potential \emph{racial discrimination} and \emph{unfairness} towards this specific demographic group. 
It is important to highlight that an individual's race should not be a determinant factor in predicting their income, and the heightened sensitivity to this feature suggests a need for careful examination and potential mitigation of biases in the models before deployment.

Building on our initial observations, we conduct an in-depth analysis of the accuracy across different racial groups for both the LR and MLP models. 
The findings, as shown in Figure \ref{fig:income_algorithmic_intervention}, align with our earlier feature stability results. 
Notably, the accuracy for the worst-performing racial group is significantly lower compared to other groups (for instance, a decrease from 82\% to 72\% in the case of the MLP model).
Such findings indicate that both the LR and MLP models, when trained using ERM, exhibit unfairness towards minority racial groups. 
In light of these insights, our feature stability analysis serves as a valuable tool to identify and prevent the deployment of models that may perpetuate such disparities in practice.

Subsequently, we use adversarial training as an algorithmic intervention to enhance model performance. 
Figure \ref{fig:income_algorithmic_intervention} illustrates the results of this intervention: AT denotes adversarial training that perturbs \emph{all} racial features, whereas targeted AT specifically perturbs the \emph{identified} sensitive racial feature ``American Indian".  
The results indicate that targeted AT markedly outperforms all baseline models, achieving a significant improvement in accuracy for the worst-performing racial group. 
This outcome effectively demonstrates the utility of our feature stability analysis in guiding targeted improvements to model performance and fairness.\\

\newpage
\noindent\textbf{Public coverage prediction}
\vspace{1mm}

\noindent We replicated the aforementioned experiment on the ACS PubCov dataset, which involves predicting an individual's public health insurance status. Following the previous setup, we identify and display the top-5 most sensitive features for both LR and MLP models in Figure \ref{fig:pubcov_feature}. 
Additionally, Figure \ref{fig:pubcov_algorithmic_intervention} presents the accuracy for the worst-performing racial group for each method.

The findings reveal several key insights:
\textrm{(i)} The MLP model outperforms the LR model in this context (71\% vs. 67\%), and it exhibits less sensitivity to input features. 
This observation suggests that feature sensitivity is influenced by both the nature of the task and the characteristics of the model.
\textrm{(ii)} Consistent with previous results, the ``American Indian" racial feature is identified as sensitive in both models. 
The accuracy of the worst-performing racial group further underscores the presence of discrimination against minority groups.
\textrm{(iii)} Leveraging our feature stability analysis, targeted AT achieves the most notable improvement. 
This again underscores the effectiveness of our evaluation method in enhancing model performance and fairness.

\vspace{2mm}
\noindent\textbf{COVID-19 mortality prediction}

\vspace{2mm}

\noindent We use the COVID-19 dataset, and the task is to predict the mortality of a patient based on features including comorbidities, symptoms, and demographic characteristics.
For the LR and MLP models trained with ERM, we follow the outlines in Section \ref{subsec:feature-analysis} and identify the top-5 most sensitive features, as shown in Figure \ref{fig:covid_feature}.
From the results, we observe that: 
\textrm{(i)} Consistent with the trends observed in the income prediction task, the LR model demonstrates lower sensitivity to input features compared to the MLP model when their performance levels are comparable;
\textrm{(ii)} Notably, both LR and MLP models are quite sensitive to the ``Age" feature. 
Given the variety of risk factors for COVID-19, such as comorbidities and symptoms, it is concerning that these models might overemphasize age, which is not the sole determinant of mortality. 
This highlights a critical need to ensure models effectively account for diverse age groups and do not rely excessively on age as a predictive factor.

\begin{figure}[htbp]
  \centering
\subfloat[Feature Stability]{\includegraphics[width=0.45\textwidth]{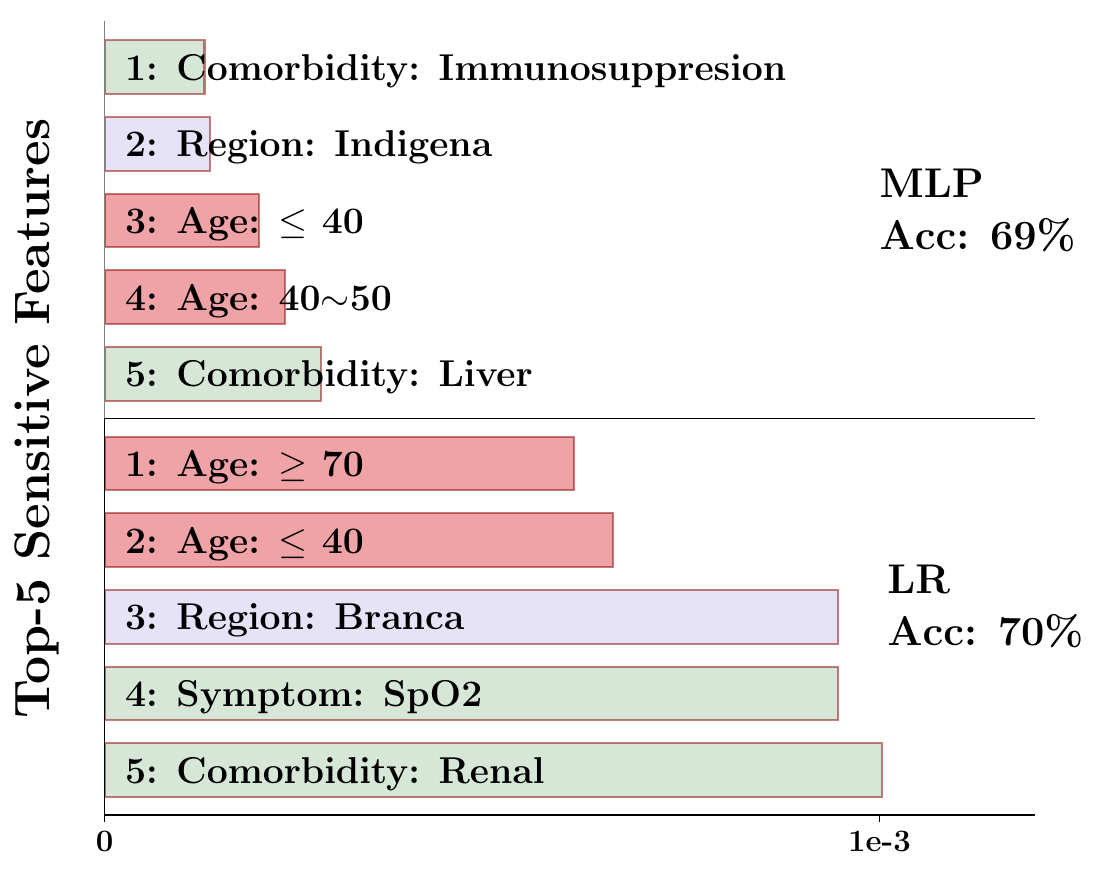}\label{fig:covid_feature}}
  \subfloat[Accuracy \& Macro F1 Score]{\includegraphics[width=0.45\textwidth]{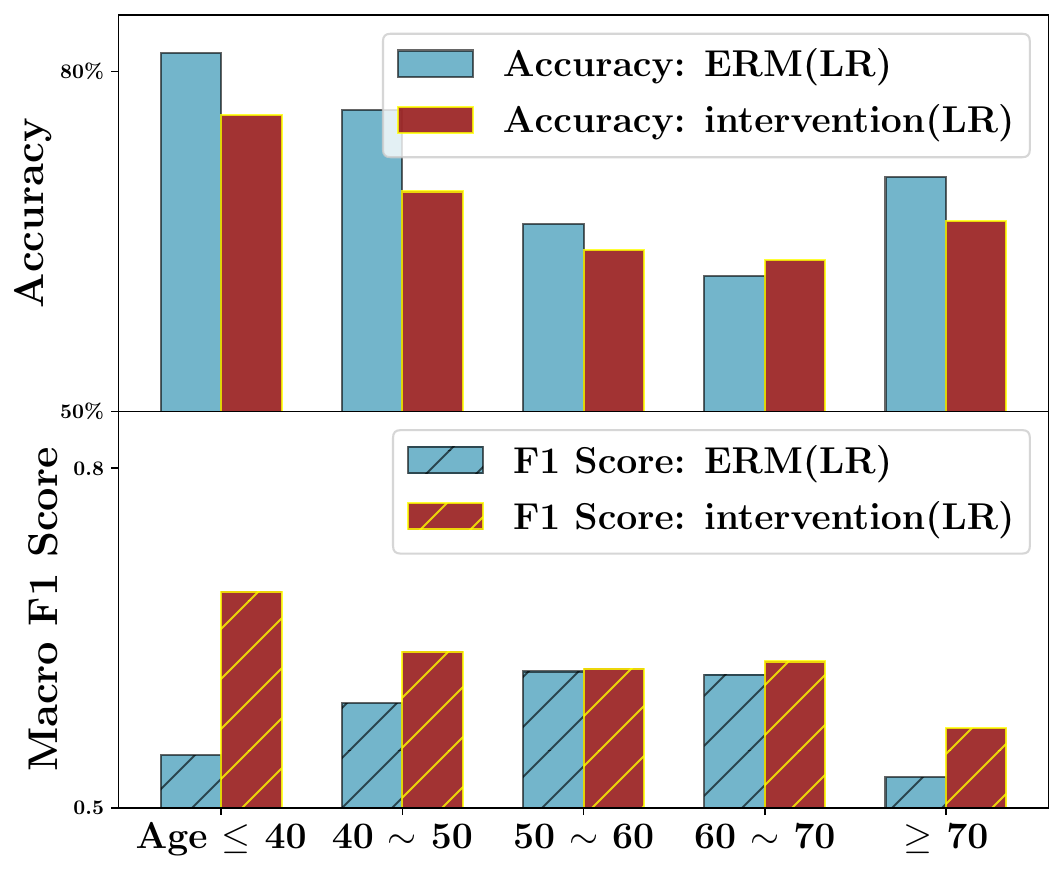}\label{fig:covid_acc}}

  \caption{Results of the COVID-19 mortality prediction task. (a): The top-5 most sensitive features for MLP and LR, respectively. (b): The prediction accuracy (upper sub-figure) and macro F1 score (lower sub-figure) before and after algorithmic intervention on the LR model.}
 \label{fig:covid_prediction}  
\end{figure}
Building on these insights, we further evaluate the accuracy and macro F1 score across different age groups for the LR model. 
As illustrated in Figure \ref{fig:covid_prediction}, the accuracy for younger individuals (age $<$ 40) and older individuals (age $\geq$ 70) is notably high (the blue bars in the upper sub-figure). 
However, their corresponding macro F1 scores are significantly lower (as shown by the blue bars in the lower sub-figure). 
This suggests that the LR model may overly rely on the age feature for making predictions. 
For example, it tends to predict survival for younger individuals and mortality for older individuals with high probability, irrespective of other relevant clinical indicators. 
Such a simplistic approach raises concerns about the model's ability to provide nuanced predictions for these age groups.

Considering the possibility of varied mortality prediction mechanisms among different age groups, we propose a targeted algorithmic intervention: training \emph{distinct} LR models for each age group. 
From the lower sub-figure in Figure \ref{fig:covid_prediction}, we  see a substantial improvement in macro F1 scores for both younger and older populations.

From these three real-world experiments, we demonstrate how the proposed feature stability analysis can help discover potential discrimination and inform targeted algorithmic interventions to improve the model's reliability and fairness.

\section{Closing Remarks}

This work proposes an OT-based stability criterion that allows both  data corruptions and sub-population shifts within a single framework. Applied to three real-world datasets, our method yields insightful observations into the robustness and reliability of machine learning models, and suggests potential algorithmic interventions for further enhancing model performance.
The utility of our stability evaluation criterion to modern model architectures (e.g., Transformer, tree-based ensembles) and popular real-world applications 
(e.g., LLMs) is natural to further explore. 


\section*{Impact Statements}
In this paper, we propose an OT-based stability criterion that addresses the challenges posed by both data corruptions and sub-population shifts, offering a comprehensive approach to evaluating the robustness of machine learning models.
The potential broader impact of this work is significant, particularly in providing a principle approach to evaluate fairness and reliability of models deployed in real-world scenarios (based on specified criteria which we take as given).  
By enabling more nuanced assessments of model stability, our criterion can help prevent the deployment of biased or unreliable models, thereby contributing to more equitable outcomes, especially in high-stakes applications like healthcare, finance, and social welfare. 
Furthermore, our work underscores the necessity of considering and mitigating potential biases and unfairness in automated decision-making systems. 
As machine learning continues to play an increasingly integral role in societal functions, the tools and methodologies developed in this study provide crucial steps towards ensuring that these technologies are used responsibly and ethically.

\bibliography{example_paper.bib}
\bibliographystyle{apalike}

\newpage
\newpage
\appendix
\onecolumn

\section{Proofs}
\subsection{Proof of Theorem \ref{thm:duality}}
\label{appendix-sec:strong-duality}
\begin{proof}

To start with, we first reformulation the primal problem \eqref{eq:primal}  into an infinite-dimensional linear program: 
\begin{equation}
\label{eq:primal_general}
    \begin{array}{cll}
     \mathop{\inf}\limits_{\pi} & \EE_{\pi}[c((Z,W),(\hat Z, \hat W))] \\
        \st & \pi \in  \mathcal{P}((\mc Z \times \mc W)^2) \\ 
        & r-\EE_{\pi} [W \cdot \ell(\beta, Z)] \leq 0\\ 
        & \EE_{\pi}[W] =1      \\ 
        & 
        \pi_{(\hat Z,\hat W)} = \hat \P. 
        \tag{Primal}
    \end{array}
\end{equation}

We aim to apply Sion's minimax theorem to the Lagrangian function
\begin{equation*}
    L(\pi;h,\alpha) = h r + \alpha +\EE_{\pi}[c((Z, W), (\hat Z, \hat W))-h\cdot W\cdot \ell(\beta,Z)-\alpha \cdot W],
\end{equation*}
where $h \in \R_+$, $\alpha \in \R$, and $\pi$ belongs to the primal feasible set 
\[
    \Pi_{\hat \P}= \left\{\pi \in \mc P((\mc Z \times \mc W)^2) \; :\; \pi_{(\hat Z, \hat W)}  =\hat \P \right\}.
\]
Since $\mathcal{Z}\times\mathcal{W}$ is compact, it follows that $\mathcal{P}(\mathcal{Z}\times\mathcal{W})$ is tight. Furthermore, as a subset of a tight set is also tight, we conclude that $\Pi_{\hat{\P}}$ is tight as well. Consequently, according to Prokhorov's theorem \citep[Theorem 2.4]{van2000asymptotic}, $\Pi_{\hat{\P}}$ has a compact closure. By taking the limit in the marginal equation, we observe that $\Pi_{\hat{\P}}$ is weakly closed, establishing that $\Pi_{\hat{\P}}$ is indeed compact. Moreover, it can be readily demonstrated that $\Pi_{\hat{\P}}$ is convex.

The Lagrangian function $L(\pi;h,\alpha)$ is  linear in both $\pi$ and $(h,\alpha)$. To employ Sion's minimax theorem, we will now prove that (i) $L(\pi;h,\alpha)$ is  lower semicontinuous in $\pi$ under the weak topology and (ii) continuous in $(h,\alpha)$ under the uniform topology in $\R_+ \times \R$. 

\textrm{(i)} Suppose  that $\pi_n$ converges weakly to $\pi$. 
    Then,  Portmanteau theorem states that for any lower semicontinuous function $g$ that is bounded below, we have
    \[
    \mathop{\lim\inf}_{n\rightarrow +\infty}\int g\,  \diff \pi_n \ge  \int g \, \diff \pi. 
    \]
   Since $\ell(\beta,\cdot)$ is upper semicontinuous  for all $\beta$ and $w,h\ge 0$, we can conclude that $h\cdot w\cdot \ell(\beta,z)$ is upper semicontinuous w.r.t $(z,w)$. Moreover, armed with the lower semicontinuity of the function  $c((z, w),(\hat z, \hat w))$, we know the following candidate function 
    \[ 
     c((z, w), (\hat z, \hat w))-h\cdot w\cdot \ell(\beta,z)-\alpha\cdot w
    \]
    is lower semicontinuous with respect to $(z,w)$ for any $(\hat z, \hat w) \in \mc Z \times \mc W$. As $\mc Z \times \mc W$ is compact, the above candidate function is also bounded below. Thus, we have
    \[
    \mathop{\lim\inf}_{n\rightarrow +\infty} L(\pi_n;h,\alpha)\ge  L(\pi;h,\alpha).\]
    It follows that $L(\pi; h, \alpha)$ is lower semicontinuous in~$\pi$ under the weak topology.
    
\textrm{(ii)} Suppose now that $\lim_{n\rightarrow +\infty}h_n= h$ in the Euclidean topology and $\lim_{n\rightarrow \infty}\alpha_n=\alpha$ in the Euclidean topology. There exists $\bar h \in \R_+$ and $\bar \alpha \in \R$ with $\sup_{n \to \infty} |h_n|\leq \bar h $ and  $\sup_{n \to \infty} |\alpha_n| <\bar \alpha$ for all $n\geq 1$. Thus, by the dominated convergence theorem, we have 
    \[
    \mathop{\lim}_{n\rightarrow +\infty} L(\pi;h_n,\alpha_n)= L(\pi;h,\alpha).
    \]
    We then conclude that $L(\pi; h, \alpha)$ is continuous in~$(h,\alpha)$ under the Ecludiean topology in~$\R_+\times \R$.

 We are now prepared to utilize Sion's minimax theorem, and thus, we have:

\begin{equation}
\label{eq:sion}
\inf_{\pi\in \Pi_{\hat \P}} \sup_{h\in\R_+,\alpha \in \R}
L(\pi;h,\alpha)  = \sup_{h\in\R_+,\alpha \in \R}\inf_{\pi\in \Pi_{\hat\P}} L(\pi;h,\alpha). 
\end{equation}

Our subsequent task involves demonstrating the equivalence between the left-hand side of \eqref{eq:sion} and the primal problem \eqref{eq:primal_general}. To achieve this, we will re-express the function $L$ as follows:
\begin{align*}
L(\pi;h,\alpha)  = \EE_{\pi}[c((Z,W),(\hat Z, \hat W)] + h\left ( r- \EE_{\pi}[ W\cdot \ell(\beta,Z)] \right)+\alpha (1-\EE_\pi[W]). 
\end{align*}
Then, we can see $\inf_{\pi\in \Pi_{\hat \P}} \sup_{h\in\R_+,\alpha \in \R}
L(\pi;h,\alpha) $ is bounded above. To start with, we construct a single support distribution as follows:
$\Q_0 = \delta_{(z^\star,1)}$ where $z^\star = \mathop{\arg\max}_{z\in\mc Z} \ell(\beta,z)$. Then, we have
\begin{align*}
\inf_{\pi\in \Pi_{\hat \P}} \sup_{h\in\R_+,\alpha \in \R}
L(\pi;h,\alpha)  & \leq  \sup_{h\in\R_+,\alpha \in \R}  L( \Q_0 \otimes \hat \P;h,\alpha), \\
&= \mathbb{E}_{\Q_{0} \otimes \hat{\P}}[c((Z,W),(\hat Z, \hat W))] + \sup_{h\in\R_+} h(r-\bar r) <+\infty , 
\end{align*}
where the second inequality follows from $\EE_{\Q_0}[W] =1$ and  the last equality holds as we know $r\leq \bar r = \EE_{\Q_0}[\ell(\beta,Z)] = \max_{z\in Z}\ell(\beta,Z)$ and $c$ is continuous and hence bounded on a compact domain $\mc Z \times W$. For any feasible point $\pi \in \Pi_{\hat \P}$, let us consider the inner supremum of the left-hand-side of \eqref{eq:sion}, ensuring it doesn't go to infinity. In this case, we find that
\begin{align*}
& r- \EE_{\pi}[ W\cdot \ell(\beta,Z)] \leq 0 \\
& \EE_{\pi}[W] =1.
\end{align*}
It remains to be shown that the sup-inf part is equivalent to the dual problem \eqref{eq:dual}. To do this, we rewrite the dual problem as 
\begin{align*}
& \sup_{h\in\R_+,\alpha \in \R}\inf_{\pi\in \Pi_{\hat\P}} L(\pi;h,\alpha).  \\
=& \sup_{h\in\R_+,\alpha \in \R} h r + \alpha + \inf_{\pi \in \Pi_
{\hat\P}}\EE_{\pi}[c((Z, W), (\hat Z, \hat W))-h\cdot W\cdot \ell(\beta,Z)-\alpha \cdot W].
\end{align*}
The last step is to take the supremum of $L$ over $\pi\in\Pi_{\hat \P}$. That is,
\begin{align*}
\ & \inf_{\pi \in \Pi_
{\hat\P}}\EE_{\pi}[c((Z, W), (\hat Z, \hat W))-h\cdot W\cdot \ell(\beta,Z)-\alpha \cdot W] \\
= \ & \EE_{\hat \P} \left [\min\limits_{(z,w) \in \mc Z \times \mc W } c((z, w), (\hat Z, \hat W))-h\cdot w \cdot \ell(\beta,z)-\alpha \cdot w \right],
\end{align*}
 due to the measurability of functions of the form $\min_{(z,w) \in \mc Z \times \mc W } c((z, w), (\hat Z, \hat W))-h\cdot w \cdot \ell(\beta,z)-\alpha \cdot w$, following the similar argument  in  \citep{blanchet2019quantifying}. 
\end{proof}

When the reference measure is discrete, i.e., $\P_0= \frac{1}{n} \sum_{i=1}^n \delta_{\hat z_i}$, we can get the strong duality result under some mild conditions. 
%

\begin{theorem}[Strong duality for problem \eqref{eq:primal}] 
  Suppose that 
 $r< \EE_{\PP_0}[\ell(\beta,Z)]$ holds. 
Then we have, 
\begin{equation}
	\mathfrak R(\beta,r) = \sup_{h \in \R_+, \alpha \in \R} hr +\alpha +\\ \EE_{\hat \P}\left[\tilde{\ell}_{c}^{\alpha, h}(\beta,(\hat Z,\hat W) )\right]. 
 \tag{D}
\end{equation}
\end{theorem}

\begin{proof}
To start, we have the primal problem \eqref{eq:primal} admits
\begin{equation}
    \begin{array}{cll}
    \mathfrak R(\beta,r)=   
     \mathop{\inf}\limits_{\pi:\pi_{(\hat Z,\hat W)} = \hat \P, \pi\in \mathcal{P}((\mathcal{Z}\times \mathcal{W})^2)} & \EE_{\pi}[c((Z,W),(\hat Z, \hat W))] \\
        \st & r-\EE_{\pi} [W \cdot \ell(\beta, Z)] \leq 0\\ 
        & \EE_{\pi}[W] =1.      
    \end{array}
\end{equation}
Due to the condition
 $r< \EE_{\PP_0}[\ell(\beta,Z)]$,
  we know the Slater condition, i.e., $
r-\EE_{\PP_0 \otimes \delta_1 \times \PP_0\otimes \delta_1} [W \cdot \ell(\beta, Z)] < 0
$, holds. Thus, we can apply \citet[Proposition 3.4]{shapiro2001duality} to get the strong duality result directly. That is, 
\begin{equation*}
	\mathfrak R(\beta,r) = \sup_{h \in \R_+, \alpha \in \R} hr +\alpha +\\ \EE_{\hat \P}\left[\inf_{(z,w) \in \mc Z \times \mc W}  c((z,w),(\hat Z, \hat W)) +\alpha w - h\cdot  w \cdot \ell(\beta,
 z)\right]. 
\end{equation*}

\end{proof}

\subsection{Proof of Proposition \ref{theorem:dual-reformulation}}
\begin{proof}
 Now, we are trying to calculate the surrogate function with our proposed cost function $c$ in \eqref{eq:cost_f} . Then, we have 
 \begin{align*}
\tilde{\ell}_{c}^{\alpha,h}(\beta,(\hat z,\hat w) ) = & \min_{(z,w)\in \mc Z\times \mc W }  \theta_1 \cdot w\cdot d(z,\hat z) + \theta_2 (\phi(w)-\phi(\hat w) )_+ -\alpha w - h\cdot  w \cdot \ell(\beta, z) \\
= & \min_{z \in \mc Z} \theta_2 \cdot \min_{w\in\R} -w\frac{h\cdot \ell(\beta,z)-\theta_1\cdot d(z,\hat z)+\alpha }{\theta_2}+\phi(w)+\mathbb{I}_{\mc W}(w) \\
= & \min_{z\in\mc Z} - \theta_2\cdot  (\phi+\mathbb{I}_{\mc W})^*\left(\frac{h\cdot \ell(\beta,z)-\theta_1\cdot d(z,\hat z)+ \alpha}{\theta_2}\right). 
 \end{align*}
 where the first equality follows as $\hat W = 1$ almost surely and $\phi(1) =0$, and the second equality holds due to the definition of conjugate functions. 

\textrm{(i)} When $\mc W = \R_+$ and $\phi(t) = t \log t -t +1$, we know its conjugate function $(\phi+\mathbb{I}_{\R_+})^*= \exp(t)-1$. Consequently, we obtain the following: 
\begin{align*}
\tilde{\ell}_{c}^{\alpha,h}(\beta,(\hat z,\hat w)) & = \min_{z\in \mc Z} - \theta_2\cdot \exp\left(\frac{h\cdot \ell(\beta,z)-\theta_1\cdot d(z,\hat z)+\alpha }{\theta_2}\right)+\theta_2 \\
& = - \theta_2\cdot \exp\left(\frac{\max_{z\in \mc Z} h\cdot \ell(\beta,z)-\theta_1\cdot d(z,\hat z)+\alpha }{\theta_2}\right)  +\theta_2 \\
& =- \theta_2\cdot \exp\left(\frac{\ell_{h,\theta_1}(\hat z)+\alpha }{\theta_2}\right)+\theta_2 . 
\end{align*}
where the second equality follows from the fact the function $\exp(\cdot)$ is monotonically increasing. Hence, we can reformulate the dual problem \eqref{eq:dual} as 
\[
\mathfrak R(\beta,r) = \sup_{h \in \R_+, \alpha \in \R} hr +\alpha -\theta_2 \EE_{\P_0}\left[\exp\left(\frac{\ell_{h,\theta_1}(\hat Z)+\alpha }{\theta_2}\right)\right]+\theta_2. 
\]
Next, we will solve the supremum problem via $\alpha$ and the first-order condition reads 
\[1-\exp \left(\frac{\alpha}{\theta_2}\right)\EE_{\P_0}\left[\exp\left(\frac{\ell_{h,\theta_1}(\hat Z)}{\theta_2}\right)\right] = 0 \] 
and $\alpha^\star =-\theta_2 \log\left(\EE_{\P_0}\left[\frac{\ell_{h,\theta_1}(\hat Z)}{\theta_2}\right]\right)$. Put all of them together, we get
\[
\mathfrak R(\beta,r) = \sup_{h \in \R_+} hr -\theta_2 \log\left(\EE_{\P_0}\left[\exp\left(\frac{\ell_{h,\theta_1}(\hat Z)}{\theta_2}\right)\right]\right). 
\]
\textrm{(ii)} 
When $\mathcal{W} = \R+$ and $\phi(t) = (t-1)^2$, the conjugate function can be computed as $(\phi + \mathbb{I}_{\R_+})^*(t) = (\frac{t}{2}+1)_+^2-1$. Additionally, it is straightforward to demonstrate that $(\phi + \mathbb{I}_{\R_+})^*(t)$ is a monotonically increasing function. Hence, we have:
\begin{align*}
\tilde{\ell}_{c}^{\alpha,h}(\beta,(\hat z,\hat w) ) = & \min_{z\in\mc Z} - \theta_2\cdot  (\phi+\mathbb{I}_{\mc W})^*\left(\frac{h\cdot \ell(\beta,z)-\theta_1\cdot d(z,\hat z)+ \alpha}{\theta_2}\right) \\ 
=&  \min_{z\in\mc Z} - \theta_2\cdot  \left(\frac{h\cdot \ell(\beta,z)-\theta_1\cdot d(z,\hat z)+ \alpha}{2\theta_2}+1\right)_+^2+\theta_2\\
= & -\theta_2\cdot \left(\frac{\ell_{h,\theta_1}(\hat z)+ \alpha}{2\theta_2}+1\right)_+^2+\theta_2
\end{align*}
where the third equality holds as the monotonicity of $(\phi + \mathbb{I}_{\R_+})^*$. Then, we can reduce the dual problme \eqref{eq:dual} as
\[\sup\limits_{h\geq 0, \alpha\in\mathbb R}  hr+\alpha+\theta_2 -\theta_2\mathbb E_{\P_0}\left[\left(\frac{\ell_{h,\theta_1}(\hat Z)+\alpha}{2\theta_2}+1\right)_+^2 \right].\]
\end{proof}
\begin{remark}
We want to highlight the distinction between the KL and $\chi^2$-divergence cases. In the latter case, we are unable to derive a closed-form expression for the optimal $\alpha^\star$. Instead, we must reduce it to a solution of a piecewise linear equation as follows:
\begin{equation}
\label{equ:alpha-star}
     \EE_{\P_0}\left[\left(\frac{\ell_{h,\theta_1}(\hat Z)+\alpha}{2\theta_2}+1\right)_+\right] =1. 
\end{equation}
\end{remark}

\subsection{Proof of Theorem \ref{thm:kl_linear}}
\label{subsec:appendix-linear}
\begin{proof}
    By introducing epigraphical
auxiliary variable~$t\in \R$, we know problem \eqref{equ:D-risk-kl} is equivalent to 
\begin{align}
	& \min_{h\geq 0} \, -hr+\theta_2 \log \mathbb E_{\P_0} \left[\exp\left(\frac{\ell_{h,\theta_1}(\hat Z)}{\theta_2}\right) \right]\notag\\
 =   & \left\{ \begin{array}{cclll}
         &\min  &-h r  +t & \\
         &\st   & h \in \R_+, t \in \R\\
         && \theta_2 \log \mathbb E_{\P_0} \left[\exp\left(\frac{\ell_{h,\theta_1}(\hat Z)}{\theta_2}\right) \right] \leq t
    \end{array}\right. \label{eq:reformu1}\\
    =  \, & \left\{ \begin{array}{cclll}
         &\min  &-h r  +t & \\
         &\st   & \lambda \in \R_+, t \in \R, \eta \in \R_+^n\\
         && (\eta_i,  \theta_2, \ell_{h,\theta_1}(\hat z_i)- t) \in \mc K_{\exp} &\forall i \in [n] \\
           && \frac{1}{n}\sum_{i=1}^n \eta_i \leq  \theta_2
    \end{array}\right. \notag \\
    = \, &\left\{ \begin{array}{cclll}
         &\min  &-h r  +t & \\
         &\st   & \lambda \in \R_+, t \in \R, \eta \in \R_+^n, p\in\R_n\\
          && (\eta_i,  \theta_2, p_i- t) \in \mc K_{\exp} &\forall i \in [n] \\
         && \ell_{h,\theta_1}(\hat z_i)\leq p_i  &\forall i \in [n]\\
             &&\frac{1}{n}\sum_{i=1}^n \eta_i \leq  \theta_2.
    \end{array}\right.\label{eq:reformu2}
\end{align}
Here,  the second equality can be derived from the fact that the second inequality in problem \eqref{eq:reformu1} can be reformulated as
\[
\mathbb E_{\P_0} \left[\exp\left(\frac{\ell_{h,\theta_1}(\hat Z)-t}{\theta_2}\right) \right]  \leq 1. 
\]
To handle this constraint, we introduce an auxiliary variable $\eta \in \R_+^n$, allowing us to further decompose it into $n$ exponential cone constraints and one additional linear constraint. Specifically, we have  
\begin{equation*}
\left\{\begin{aligned}
& \frac{1}{n} \sum_{i=1}^n \eta_i \leq \theta_2 \nonumber \\
& \theta_2 \exp\left(\frac{\ell_{h,\theta_1}(\hat z_i)-t}{\theta_2}\right) \leq \eta_i, \quad \forall i\in[n] \nonumber
\end{aligned}\right.
\end{equation*}
The third constraint can be further reduced to \eqref{eq:reformu2} by considering the fact that the set $\mathcal{K}_{\exp}$ corresponds to the exponential cone, which is defined as 
\begin{equation*}
K_{\exp} = \left\{(x_1,x_2,x_3)\in \R^3: x_1\ge x_2 \cdot \exp\left(\tfrac{x_3}{x_2}\right),x_2>0\right\} \cup \{(x_1,0,x_3)\in\R^3: x_1\ge 0, x_3\leq 0\}. 
\end{equation*}
The fourth equality is due to $\ell_{h,\theta_1}(\hat z_i)\leq p_i$ when we introduce  auxiliary variables $p_i$.

Next, we show that $\ell_{h,\theta_1}(\hat z_i)\leq p_i $ admits the following equivalent forms
\begin{equation}
\label{eq:piecelinear_sup}
\begin{aligned}
& \ell_{h,\theta_1}(\hat z_i)\leq p_i \\
\iff & \sup\limits_{z \in \mc Z} \left\{h\cdot \max\limits_{k\in[K]} y \cdot a_k^\top x +b_k \!-\!\theta_1 d(z,\hat z_i) \right\}\!\leq\! p_i \\
\iff & \sup\limits_{z \in \mc Z} \left\{h\cdot  y \cdot a_k^\top x +b_k \!-\!\theta_1 d(z,\hat z_i) \right\}\!\leq\! p_i ~\forall k\in[K]\\
\iff & \sup\limits_{x \in \R^d} \left\{h\cdot  \hat y_i \cdot a_k^\top x +b_k \!-\!\theta_1 \|x-\hat x_i\|_2^2 \right\}\!\leq\! p_i ~\forall k\in[K]\\
\iff &   \frac{\|a_k\|_2^2}{4\theta_1}\cdot h^2  + \hat y_i \cdot a_k^T\hat x_i \cdot h +b_k\leq p_i,\, ~\forall k\in[K]
\end{aligned}
\end{equation}
where the second equivalence arises from the non-negativity of $h$, while the third one can be derived from the nature of the cost function, which is defined as $d(z,\hat z_i) = \|x-\hat x_i\|_2^2 + \infty \cdot |y-\hat y_i|$. The second term in the cost function prevents us from perturbing the label due to the imposed budget limit.

Put everthing together, we have 
\begin{align*}
\begin{array}{cclll}
         &\min  &-h r  +t & \\
         &\st   & \lambda \in \R_+, t \in \R, \eta \in \R_+^n, p\in\R_n\\
          && (\eta_i,  \theta_2, p_i- t) \in \mc K_{\exp} &\forall i \in [n] \\
         && \frac{\|a_k\|_2^2}{4\theta_1}\cdot h^2  + \hat y_i \cdot a_k^T\hat x_i \cdot h +b_k\leq p_i,\,  &~\forall k\in[K], \, \forall i \in [n]\\
             &&\frac{1}{n}\sum_{i=1}^n \eta_i \leq  \theta_2.
    \end{array}
\end{align*}
\end{proof}


\subsection{Proof of Theorem \ref{thm:chi_linear}}
\label{subsec:proof-chi}
\begin{proof}
By introducing epigraphical
auxiliary variable~$t\in \R$, we know problem \eqref{equ:D-risk-kl} is equivalent to 
 \begin{align}
	&\min\limits_{h\geq 0, \alpha\in\mathbb R}  \, -hr-\alpha+\theta_2 +\theta_2\mathbb E_{\P_0}\left[\left(\frac{\ell_{h,\theta_1}(\hat Z)+\alpha}{2\theta_2}+1\right)_+^2 \right] \notag \\
 =   & \left\{ \begin{array}{cclll}
         &\min  &-h r  -\alpha+t & \\
         &\st   & h \in \R_+, \alpha \in \R, t \in \R\\
         && \theta_2\mathbb E_{\P_0}\left[\left(\frac{\ell_{h,\theta_1}(\hat Z)+\alpha}{2\theta_2}+1\right)_+^2 \right] \leq t 
    \end{array}\right.\notag\\
    =  \, & \left\{ \begin{array}{cclll}
         &\min  &-h r  +t & \\
         &\st   & h \in \R_+, \alpha \in \R, t \in \R, \eta \in \R_+^n\\
         &&  \ell_{h,\theta_1}(\hat z)+{2\theta_2}\alpha +2\theta_2 \leq 2\theta_2 \eta_i \,  &~ \forall i \in [n] \\
           && \frac{\theta_2}{n} \sum_{i=1}^n \eta_i^2  \leq t \nonumber  
    \end{array}\right. \notag \\
    = \, & \left\{ \begin{array}{cclll}
         &\min  &-h r  +t & \\
         &\st   & h \in \R_+, \alpha \in \R, t \in \R, \eta \in \R_+^n\\
         &&   \frac{\|a_k\|_2^2}{4\theta_1}\cdot h^2  + \hat y_i \cdot a_k^T\hat x_i \cdot h +b_k+{2\theta_2}\alpha +2\theta_2 \leq 2\theta_2 \eta_i\,  &~\forall k\in[K], \forall i \in [n]  \\
           && \frac{\theta_2}{n} \sum_{i=1}^n \eta_i^2  \leq t \nonumber  
    \end{array}\right.\label{eq:reformu2}
\end{align}
Here, the second equality follows from the fact that the constraint can be reformulated as 
\begin{equation*}
\left\{\begin{aligned}
& \frac{\theta_2}{n} \sum_{i=1}^n \eta_i^2  \leq t \nonumber \\
& \ell_{h,\theta_1}(\hat z)+{2\theta_2}\alpha +2\theta_2 \leq 2\theta_2 \eta_i, \eta_i\in\R_+. 
\end{aligned}\right.
\end{equation*}
as the function $(\cdot)_+^2$ is monotonically increasing. The last equality holds due to \eqref{eq:piecelinear_sup}
. \end{proof}

\section{Pseudo-code for Algorithms}
In this section, we provide the pseudo-code of our algorithms.
For $\phi(t)=t\log t-t+1$, please refer to Algorithm \ref{algo:1}, and for $\phi(t)=(t-1)^2$, please see Algorithm \ref{algo:2}.

\begin{algorithm*}[htbp]
\caption{Stability evaluation with general nonlinear loss functions ($\phi(t)=t\log t-t+1$).}
\label{algo:1}
\begin{algorithmic}[1]
\STATE {\bf Input}: trained model $f_\beta(\cdot)$, samples $\{\hat z_i\}_{i=1}^n$, adjustment parameters $\theta_1,\theta_2$, pre-defined threshold $r$;
\STATE {\bf Hyper-parameters}: outer iteration number $T_{\textrm{out}}$, inner iteration number $T_{\textrm{in}}$, learning rates $\eta, \gamma$;
\STATE {\bf Initialize} for $i\in [n]$, set $z_i^{(0)}\leftarrow \hat z_i$, and $h^{(0)}=1$;
\FOR{$t=0$ to $T_{\textrm{out}}-1$}
	\FOR{$k=0$ to $T_{\textrm{in}}-1$}
		\STATE For $i\in [n]$, $z_i^{(k+1)}\leftarrow z_i^{(k)}+\eta\cdot \nabla_Z \bigg(h^{(t)}\ell(\beta,z_i^{(k)})-\theta_1 d(z_i^{(k)}, \hat z_i)\bigg)$ (update samples using ADAM optimizer)
  	\ENDFOR
	\STATE Update the dual parameter using ADAM optimizer as: 
    \begin{equation*}
		h^{(t+1)}\leftarrow h^{(t)}+\gamma\cdot \nabla_h \bigg(h^{(t)}r- \theta_2 \log \sum_{i=1}^n \bigg[\exp(\frac{h^{(t)}\ell(\beta,z_i^{(T_{\text{in}})})-\theta_1 d(z_i^{(T_{\text{in}})},\hat z_i)}{\theta_2}) \bigg] \bigg)
	\end{equation*}    
\ENDFOR
\STATE {\bf Output}: stability criterion $\mathfrak R(\beta,r)$ (Equation \eqref{equ:D-risk-kl}), the most sensitive distribution $\hat{\Q}^*$ (according to Remark \ref{remark:q-star}).
\end{algorithmic}
\end{algorithm*}

\begin{algorithm}[htbp]
\caption{Stability evaluation with general nonlinear loss functions ($\phi(t)=(t-1)^2$)}
\label{algo:2}
\begin{algorithmic}[1]
\STATE {\bf Input}: trained model $f_\theta(\cdot)$, samples $\{\hat z_i\}_{i=1}^n$, adjustment parameters $\theta_1,\theta_2$, mis-classification threshold $r$;
\STATE {\bf Hyper-parameters}: outer iteration number $T_{\textrm{out}}$, inner iteration number $T_{\textrm{in}}$, learning rates $\eta, \gamma_h, \gamma_\alpha$;
\STATE {\bf Initialize} for $i\in [n]$, set $z_i^{(1)}\leftarrow \hat z_i$, and $h^{(1)}=1$;
\FOR{$t=1$ to $T_{\textrm{out}}$}
	\FOR{$k=1$ to $T_{\textrm{in}}$}
		\STATE For $i\in [n]$, $z_i^{(k+1)}\leftarrow z_i^{(k)}+\eta\cdot \nabla_Z \bigg(h^{(t)}\ell(\beta; z_i^{(k)})-\theta_1 d(z_i^{(k)}, \hat z_i)\bigg)$; \\ (update samples using ADAM optimizer)
	\ENDFOR
	\STATE Compute $\alpha^*$ via Equation \ref{equ:alpha-star};
	\STATE Update the dual parameter using ADAM optimizer as: 
    \begin{equation}
        h^{(t+1)}\leftarrow h^{(t)}+\gamma\cdot \nabla_h \bigg(  hr+\alpha^*+\theta_2 -\theta_2\sum_{i=1}^n \left(\frac{\ell_{h,\theta_1}(\beta,z_i^{(T_{\text{in}})})+\alpha^*}{2\theta_2}+1\right)_+^2    \bigg);       
    \end{equation}
\ENDFOR
\STATE {\bf Output}: stability criterion $\mathfrak R(\beta,r)$ (Equation \eqref{equ:D-risk-kl}), the most sensitive distribution $\hat{\Q}^*$ (according to Remark \ref{remark:q-star}).
\end{algorithmic}
\end{algorithm}

\end{document}